\documentclass{article}

\usepackage{microtype}
\usepackage{graphicx}
\usepackage{subcaption}
\usepackage{booktabs} % for professional tables

\usepackage{hyperref}

\PassOptionsToPackage{numbers, compress}{natbib}
\usepackage[preprint]{neurips_2026}
\usepackage[utf8]{inputenc}
\usepackage[T1]{fontenc}
\usepackage{url}
\usepackage{xcolor}

\usepackage{amsmath}
\usepackage{amssymb}
\usepackage{mathtools}
\usepackage{amsthm}
\usepackage{multirow}

\usepackage[capitalize,noabbrev]{cleveref}

\usepackage{comment}
\usepackage{macros}
\usepackage{amsmath,amsfonts,bm}
\usepackage{amssymb, amsthm,mathtools}

\def\eqref#1{equation~\ref{#1}}
\def\1{\bm{1}}

\DeclareMathAlphabet{\mathsfit}{\encodingdefault}{\sfdefault}{m}{sl}
\SetMathAlphabet{\mathsfit}{bold}{\encodingdefault}{\sfdefault}{bx}{n}

\usepackage{colortbl}

\usepackage{adjustbox}
\usepackage{titletoc}
\usepackage{dirtytalk}
\usepackage{makecell}
\usepackage[most]{tcolorbox} % loads common libraries incl. skins, breakable, etc.

\definecolor{mapcolorRGB}{RGB}{101,114,156}
\colorlet{mapcolor}{mapcolorRGB}

\definecolor{dakrpurpletag}{RGB}{108,72,197}
\definecolor{lightpurpletag}{RGB}{205,193,255}
\newtcbox{\captag}{
  enhanced,
  nobeforeafter,
  tcbox raise base,
  boxrule=0.5pt,
  top=0mm,
  bottom=0mm,
  right=0mm,
  left=0mm,
  arc=2pt,
  boxsep=1.5pt,
  before upper={\vphantom{dlg}},
  colframe=dakrpurpletag,
  coltext=mapcolor!25!black,
  colback=lightpurpletag
}

\definecolor{lightgraytag}{RGB}{242,242,242}
\newtcbox{\alttag}{
  enhanced,
  nobeforeafter,
  tcbox raise base,
  boxrule=0.5pt,
  top=0mm,
  bottom=0mm,
  right=0mm,
  left=0mm,
  arc=2pt,
  boxsep=1.5pt,
  before upper={\vphantom{dlg}},
  colframe=mapcolor!50!black,
  coltext=mapcolor!25!black,
  colback=lightgraytag % ← your new background
}

\theoremstyle{plain}
\newtheorem{theorem}{Theorem}[section]
\theoremstyle{definition}
\newtheorem{definition}[theorem]{Definition}

\theoremstyle{remark}

\newtheorem{observation}[theorem]{Observation}

\usepackage[textsize=tiny]{todonotes}

\usepackage{multirow}
\usepackage{caption}

\begin{document}

\title{A Geometric Unification of Concept Learning with Concept Cones}
%\author{Anonymous Authors}

\definecolor{primary}{RGB}{49, 46, 129}
\definecolor{secondary}{RGB}{49, 46, 129}
\newcommand{\auth}[1]{\textbf{#1}}
\newcommand{\affil}[1]{{\small #1}}
\newcommand{\amail}[1]{{\small\texttt{#1}}}
\newcommand{\afn}[1]{\textcolor{primary}{$^{#1}$}}

\author{
\auth{Alexandre Rocchi}{\textcolor{secondary}{$^\star$}\afn{a}} \quad
\auth{Thomas Fel}{\textcolor{secondary}{$^\star$}\afn{b}} \quad 
\auth{Gianni Franchi}{\textcolor{secondary}{}\afn{a}} \quad 
\vspace{4pt}\\
\afn{a}\affil{AMIAD} \quad
\afn{b}\affil{Kempner Institute, Harvard University} \quad 
%\afn{c}\affil{Harvard University}
\vspace{3mm}\\
{
\small 
\textcolor{primary}{\texttt{\{alexandre.rocchi-henry,gianni.franchi\}@polytechnique.edu}}
%alexandre.rocchi@ensta.fr
%gianni.franchi@ensta.fr
%tfel@seas.harvard.edu
}
\vspace{-4mm}
}

\maketitle

\begin{abstract}

We study two main approaches to interpretability: Concept Bottleneck Models (CBMs), which define concepts using supervision, and Sparse Autoencoders (SAEs), which discover them from data. While CBMs use supervision to align activations with human-labeled concepts, SAEs rely on sparse coding to uncover emergent ones. We show that both paradigms instantiate the same geometric structure: 
each learns a set of linear directions in activation space whose nonnegative combinations form a \emph{concept cone}.  Supervised and unsupervised methods thus differ not in kind but in how they select this cone.
SAEs are known to be unstable: small changes in sparsity, dictionary size, or training can lead to very different decompositions with similar reconstruction. This makes it difficult to assess whether the learned concepts are meaningful. We address this issue by using CBMs as partial anchors. Rather than requiring full agreement, we check whether a subset of SAE concepts can recover CBM concepts.
We formalize this idea through metrics that measure containment and alignment between cones. These metrics reveal how inductive biases—such as sparsity and expansion factor—affect the emergence of meaningful concepts, and identify regimes where SAEs best recover human-relevant structure. Overall, our work unifies supervised and unsupervised concept discovery through a shared geometric framework, providing principled metrics to measure SAE progress and assess how well discovered concepts align with plausible human concepts.

\end{abstract}
    
\section{Introduction}

As artificial intelligence systems reach superhuman performance across many tasks~\citep{lecun2015deep,Serre2019DeepLT}, understanding how they represent and process information has become a central challenge~\citep{saeed2023explainable,alecu2022can}. This is not only a scientific question: in domains such as healthcare~\citep{vamathevan2019applications}, autonomous driving~\citep{grigorescu2020survey}, and scientific discovery~\citep{jumper2021highly}, trust depends on the ability to interpret model decisions~\citep{doshivelez2017rigorous,gilpin2018explaining}. Vision models are no exception. With architectures such as Vision Transformers~\citep{dosovitskiy2020image,dehghani2023scaling,zhai2022scaling,oquab2023dinov2} reaching human-level performance, understanding their internal representations has become a key research goal.

Early interpretability methods focused on attributing predictions to input regions~\citep{simonyan2013deep,zeiler2014visualizing,bach2015pixel,springenberg2014striving,smilkov2017smoothgrad,sundararajan2017axiomatic,selvaraju2017gradcam,fong2017meaningful,fel2021sobol,novello2022making,muzellec2023gradient}. These approaches help identify \emph{where} models attend, but provide limited insight into \emph{what} they represent. Later work showed that such methods often capture superficial sensitivities rather than the internal structure of representations~\citep{hase2020evaluating,hsieh2020evaluations,nguyen2021effectiveness,kim2021hive,sixt2020explanations,fel2021cannot}.

To go beyond this limitation, research has shifted toward \emph{concept-based interpretability}~\citep{kim2018interpretability,poeta2023concept}, which aims to identify meaningful directions in representation space~\citep{bau2017network,ghorbani2019towards,zhang2021invertible,fel2023craft,graziani2023concept,vielhaben2023multi,kowal2024understanding,kowal2024visual,fel2023holistic}. Early work tried to associate single neurons with concepts, but it is now well understood that neurons are polysemantic and basis-dependent~\citep{arora2018linear,elhage2022superposition,gorton2024missing}. This has led to a distributed view, where concepts correspond to patterns across many units rather than individual neurons~\citep{ghorbani2017interpretation,fel2023craft,zhang2021invertible}.

This perspective is supported by the \emph{Linear Representation Hypothesis} (LRH)~\citep{elhage2022superposition,park2023linear,costa2025flat}, which suggests that neural networks encode many features as sparse linear combinations of a shared set of directions. Under this view, interpretability becomes a problem of recovering this underlying dictionary~\citep{fel2023holistic}. Sparse autoencoders (SAEs)~\citep{makhzani2013k,cunningham2023sparse,bricken2023monosemanticity,joseph2025steering} operationalize this idea by learning sparse representations over an overcomplete dictionary. Applied to large models such as DINOv2~\citep{oquab2023dinov2,darcet2023vision,baldassarre2025back,fel2025archetypal}, SAEs uncover many recurring patterns that can be interpreted as concepts.

However, unsupervised concept discovery faces a fundamental limitation: \emph{non-identifiability}~\citep{klindt2023identifying,klindt2025superposition,locatello2019challenging}. Many different dictionaries can reconstruct the same activations equally well, and reconstruction loss does not tell us which one is meaningful. In practice, SAEs are unstable: small changes in sparsity, dictionary size, initialization, or training schedule lead to very different decompositions with similar reconstruction error~\citep{fel2025archetypal,paulo2025sparse}. Moreover, automatic interpretability scores can be misleading, as even random dictionaries may achieve similar values~\citep{heap2025sparse}. As a result, practitioners lack reliable tools to assess whether a learned decomposition is meaningful.

In parallel, \emph{Concept Bottleneck Models} (CBMs)~\citep{koh2020concept,
oikarinen2023label,vandenhirtz2024stochastic,alukaev2023cross,chauhan2023interactive,jeyakumar2022automatic,tan2024explain,kazmierczak2023clip,moayeri2023text,oikarinen2022clip,panousis2023discover,yuksekgonul2022post} follow a different approach. CBMs enforce interpretability by design, introducing human-defined concepts as intermediate variables. This provides explicit semantic structure, but relies on labeled concepts and predefined taxonomies that may be incomplete or misaligned with learned representations~\citep{margeloiu2021concept,debole2025if}.

Despite their differences, both approaches aim to uncover the same underlying structure. SAEs discover concepts from data, while CBMs impose them through supervision. In this work, we propose to connect these two paradigms through a common geometric framework.

\textbf{Our perspective.}
We view both CBMs and SAEs as instances of dictionary learning, where activations are decomposed into concept directions and corresponding coefficients. This leads to a unified representation in terms of \emph{concept cones}. Within this framework, we can compare supervised and unsupervised concept spaces in a principled way. However, we do not expect perfect alignment between SAEs and CBMs. CBMs capture only a limited set of human-defined concepts, while SAEs discover a much richer set of features. Our key idea is therefore the following: \emph{A decomposition should be considered more plausible with respect to a given semantic reference if a subset of SAE atoms can recover concepts that are independently validated and expected to be present.}

More precisely, CBMs provide a \emph{partial reference}: they encode a set of concepts that we hypothesize should be present. If a small subset of SAE atoms can reconstruct these concepts through sparse nonnegative combinations, this suggests that the SAE has identified a meaningful decomposition among the many possible ones allowed by non-identifiability~\citep{locatello2019challenging}. Conversely, if none of the SAE atoms align with a recognizable structure, there is reason to question the entire decomposition. This perspective turns CBMs into a practical tool for evaluating SAEs. In language models, SAEs can often be interpreted directly by linking concepts to words, making them naturally accessible to humans \citep{bricken2023monosemanticity}. In vision, however, such direct grounding is not available, and human annotation remains the standard way to associate meaning to representations. In this context, it is hard to assess SAEs, and using CBMs as a reference is not fundamentally different from existing practices, but rather provides a structured and reproducible alternative. Rather than imposing a full ground truth, they provide a \emph{partial sanity check}: they indicate whether the expected structure is present, without constraining the discovery of additional concepts.

\textbf{Contributions.}
In this work:
\begin{itemize}
\item We show that CBMs and SAEs can be understood within a common framework of dictionary learning and concept cones, providing a principled way to compare supervised and unsupervised concept representations.

\item We introduce a set of complementary metrics: coverage, sparsity, geometric alignment, activation alignment, and regression predictability, and provide theoretical results clarifying what each metric measures and why none of them is sufficient on its own.

\item We formalize the idea of \emph{partial alignment}: a good SAE is not one that matches all CBM concepts, but one where a sparse subset of atoms recovers meaningful concepts while allowing additional structure to emerge.

\item We demonstrate empirically that this framework provides actionable guidance for SAE design. Through extensive benchmarks, we analyze how key factors such as sparsity, dictionary size, and expansion ratio influence this alignment. Our results reveal a “sweet spot” of parsimony where SAEs best capture CBM-like representations.  

\end{itemize}

By combining the scalability of unsupervised methods with the semantic grounding of supervised concepts, our approach provides a practical framework for evaluating and improving interpretable representations in modern neural networks.
\section{Concept cones: a unifying geometry for concept learning}

\label{sec:theory}

In this section, we formalize a shared view of concept learning through the lens of dictionary learning.
Our aim is to show that Concept Bottleneck Models (CBMs) and Sparse Autoencoders (SAEs) solve the same inverse problem under different regularization.

\begin{figure}[t]
    \vspace{-2mm}
    \centering
    \includegraphics[width=0.45\linewidth]{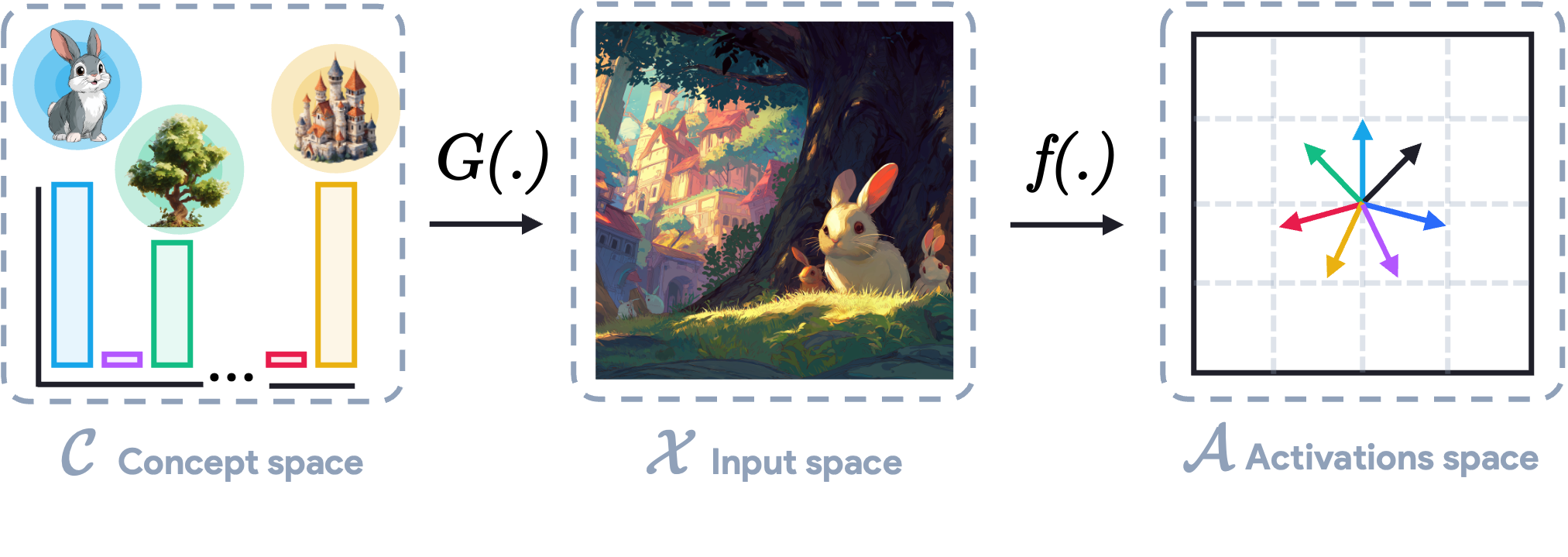}
    \vspace{-4mm}

   \caption{
\textbf{Illustration of Concept Learning Spaces.} 
Hidden factors $\bm{c} \in \mathcal{C}$ (human-labeled concept space) generate images $\bm{x} \in \mathcal{X}$ via $\dgp$. 
A network $\f$ maps these images to an activation space $\mathcal{A}$, where concepts are linearly decodable. 
Concept extraction seeks a mapping $\bm{H}: \mathcal{A} \to \mathcal{C}$ to recover the original hidden factors from the model's internal representations.
}\label{fig:concept_schema}\label{fig:concept_schema}
    \vspace{-4mm}
\end{figure}

\subsection{Theoretical Ssetup}

\paragraph{Notation.}
Let \(\SX\subset\mathbb{R}^{H\times W\times 3}\) be the input space and \(\SA\subset\mathbb{R}^{d}\) a feature space of a fixed representation map \(\f:\SX\to\SA\).
Given a dataset \(\X=(\x_1,\ldots,\x_n)\subset\SX\), define the activation matrix \(\A=\f(\X)\in\mathbb{R}^{n\times d}\) with rows \(\a_i=\f(\x_i)^{\top}\).
Let \(c\) denote the number of concepts, \(\D\in\mathbb{R}^{c\times d}\) a concept dictionary (rows are atoms), and \(\Z\in\mathbb{R}^{n\times c}\) a code matrix (rows are per-example concept coordinates).

\paragraph{Data-generating process.} %\fel{cite non linear ica people here}
We assume observations arise from a latent space of compositional factors of variation.
Let \(\mathcal{C}\subset\mathbb{R}^{K}\) (here \(K\) is the dimension of the latent space, so the number of independents attributes, e.g.: hair color, pose, lighting, expression, shapes, textures,  etc.) be a latent space with product measure \(P_{\bm{c}}=\bigotimes_{k=1}^{K}P_{k}\), where \(P_{k}\) is the probability distribution of the $k$-th latent factor, the coordinate \(c_k\).
Let \(\dgp:\mathcal{C}\to\SX\) be a measurable map that is \(C^{1}\) and injective on a full-measure subset of \(\mathcal{C}\).
The data distribution on \(\SX\) is the pushforward \(P_{\bm{x}}=\dgp_{\#}P_{\bm{c}}\).
When \(\dgp\) is a diffeomorphism onto its image and \(p_{\bm{c}}\) is a density,
\(
p_{\bm{x}}(\x) = p_{\bm{c}}(\dgp^{-1}(\x))\,\bigl|\det \bm{J}_{\dgp^{-1}}(\x)\bigr|.
\)
Each coordinate \(c_k\) encodes a part or attribute that contributes to \(\x=\dgp(\bm{c})\) through composition.

\subsection{Dictionary-learning view and cone view}
\paragraph{Unsupervised concept extraction as dictionary learning.}
As shown by previous work, most of the unsupervised concept extraction methods boils down to dictionary learning that seek to approximate the activation matrix $\A \in \mathbb{R}^{n \times d}$ via a factorization $\A \approx \Z \D$, where $\Z \in \mathbb{R}^{n \times c}$ contains per-example codes (concepts) and $\D \in \mathbb{R}^{c \times d}$ is a dictionary whose rows represent concept directions (atoms). The different methods of unsupervised concept extraction that have been proposed \textit{only differ in their constraints} :
\vspace{-6mm}

\begin{equation}
%\vspace{-4mm}
\footnotesize
\begin{aligned}\label{eq:dico_constraints}
%\nonumber
    & (\Z^\star, \D^\star) = \argmin_{\Z, \D} || \A - \Z \D ||^2_F, ~~\text{subject to} 
    ~~ &\begin{cases}
        \forall i, \Z_i \in \{ \e_1, \ldots, \e_k \}, & \text{(\textbf{ACE} \citet{ghorbani2019towards})}, \\
        \D \D^\tr = \mathbf{I}, & \text{(\textbf{ICE} \citet{zhang2021invertible})}, \\
        \Z \geq 0, \D \geq 0, & \text{(\textbf{CRAFT} \citet{fel2023craft})}, \\
        \D_i \in \conv(\A) & \text{(\textbf{Archetypal Analysis} \citet{fel2025archetypal})}, \\
        \Z = \bm{\Psi}_{\theta}(\A), ||\Z||_0 \leq K, & \text{(\textbf{SAE}s) \citet{bricken2023monosemanticity}}.
    \end{cases}
\end{aligned}
\end{equation}

where \(\bm{\Psi}(\cdot)\) is generally a linear encoder equipped with a ReLU activation and a sparsity projection (such as TopK~\citep{gao2024scaling}, BatchTopK~\citep{bussmann2024batchtopk}, or JumpReLU~\citep{hindupur2025projecting}).  
Across all these formulations, the learned codes \(\Z\) are nonnegative or zero, meaning that activations are reconstructed through additive, parts-based combinations of concept directions.

\paragraph{Supervised concept extraction with linear bottlenecks.}
Having cast unsupervised methods as dictionary learning, we now show that Concept Bottleneck Models (CBMs) also fit a regular template that we should clarify now. Essentially, it adds supervision on the codes to align them with labeled concepts. 
For activations \(\a_i \in \SA\), a CBM predicts concept scores \(\hat{\bm{c}}_i\) and labels \(\hat{\bm{y}}_i\) through a linear bottleneck. 
Most CBM variants seek to learn a dictionary \(\mathbf{W}_c \in \mathbb{R}^{c \times d}\) whose rows represent concept directions and codes (concept coefficient) \(\hat{\bm{c}}_i \in [0,1]^c\) that are predicted from activations via \(\hat{\bm{c}}_i = \sigma(\a_i \mathbf{W}_c^\top + \bm{b}_c)\), followed by a linear classifier \(\hat{\bm{y}}_i = \mathrm{softmax}(\hat{\bm{c}}_i \mathbf{W}_y^\top + \bm{b}_y)\). 
The different CBMs that have been proposed \textit{only differ in their supervision and training}:
\begin{equation}
\footnotesize % \small
\begin{aligned}\label{eq:cbm_constraints}
& (\mathbf{W}_c^\star, \mathbf{W}_y^\star) = \argmin_{\mathbf{W}_c, \mathbf{W}_y} \sum_{i=1}^n \mathcal{L}_{\mathrm{CE}}(\hat{\bm{y}}_i, \bm{y}_i) + \lambda \mathcal{L}_{\mathrm{BCE}}(\hat{\bm{c}}_i, \bm{c}_i), 
&~~~\text{subject to}\\
~~ &\begin{cases}
    \lambda > 0, \text{ joint training}, & \text{(\textbf{Joint CBM} \citet{koh2020concept})}, \\
    \lambda > 0, \text{ sequential training}, & \text{(\textbf{Sequential CBM} \citet{koh2020concept})}, \\
    \lambda = 0, \mathbf{W}_c = \D_{\texttt{LLM}}, & \text{(\textbf{Label-Free CBM} \citet{oikarinen2023label})}, \\
    \lambda = 0, \mathbf{W}_c = \D_{\texttt{Bank}}, & \text{(\textbf{Post-hoc CBM} \citet{yuksekgonul2022post})}, \\
    \lambda = 0, \mathbf{W}_c = \D_{\texttt{SAE}}, & \text{(\textbf{DN CBM} \citet{rao2024discover})}, \\
    \hat{\bm{c}}_i \sim \mathcal{N}(\sigma(\bm{\mu}(\a_i)), \bm{\Sigma}(\a_i)), & \text{(\textbf{Stochastic CBM} \citet{vandenhirtz2024stochastic})}, \\
    \lambda = 0, \mathbf{W}_c = \D_{\texttt{Clip}}, & \text{(\textbf{Clip-QDA} \citet{kazmierczak2023clip})}.
    \end{cases}
\end{aligned}
\end{equation}

Here \(\D_{\texttt{LLM}}\) denotes concepts generated by large language models (e.g., GPT-3) and embedded via CLIP, \(\D_{\texttt{Bank}}\) represents predefined concept banks or post-hoc extracted directions (e.g., via \textit{Broden} dataset), \(\D_{\texttt{SAE}}\) refers to dictionary atoms learned by SAEs, and \(\D_{\texttt{CLIP}}\) denotes principal components of CLIP text embeddings.
Note that some methods introduce architectural variations: Post-hoc CBM may include residual connections, CLIP-QDA uses quadratic discriminant analysis instead of linear classification, and Stochastic CBM samples concept activations from learned distributions.
Independent, sequential, and joint training regimes correspond to different optimization schedules on the same objective, thus preserving the bottleneck geometry.
For detailed descriptions of each variant, see Appendix \ref{app:cbm_variants}.

Now that we have studied the two different frameworks of unsupervised and supervised concept extraction,  a natural question arises: \textit{what structure unifies these seemingly distinct approaches?}
We now show that a possible answer is a specific geometric structure that we will now describe.

%\subsection{Cone geometry}
\paragraph{Concept cones: a unifying geometric framework.}
A key observation unifying supervised and unsupervised approaches is that the activation function of the CBM ensures all concept codes (coefficients) satisfy \(\bm{c} \geq 0\), exactly as in unsupervised dictionary learning where codes are constrained to be nonnegative (Equation~\ref{eq:dico_constraints}).
Thus both supervised methods (CBMs) and unsupervised methods (SAEs, CRAFT, etc.) produce \textit{nonnegative codes} and learn a dictionary whose rows encode concept directions—\(\mathbf{W}_c\) for CBMs and \(\D\) for unsupervised extractors.
The geometric consequence is that the representable activation subspace is characterized by the \emph{conic hull} generated by these concept directions.
This conical structure provides a unified view of concept extraction across approaches.

\begin{definition}[Concept Cone]
\label{def:concept_cone}
Given a concept dictionary \(\D \in \mathbb{R}^{c \times d}\) (equivalently \(\mathbf{W}_c\) for CBMs), we define the associated \emph{concept cone} as
\begin{equation}
\begin{aligned}
\mathcal{C}_{\D} &= \left\{ \bm{v} \in \mathbb{R}^d ~:~ \bm{v} = \bm{\alpha}^\top \D \text{ for some } \bm{\alpha} \in \mathbb{R}_+^c \right\} 
&= \mathrm{cone}(\D_1, \ldots, \D_c),
\end{aligned}
\label{eq:concept_cone}
\end{equation}
where \(\D_i \in \mathbb{R}^d\) denotes the \(i\)-th row (concept direction) of \(\D\), and \(\mathrm{cone}(\cdot)\) denotes the conic hull.
\end{definition}

This geometric perspective yields several important insights formalized in the following observation.

\begin{observation}[Concept Cones Unify Extraction Methods]
\label{prop:concept_cone_unification}
Let \(\D \in \mathbb{R}^{c \times d}\) be a concept dictionary learned by any method in ~\cref{eq:dico_constraints} or \cref{eq:cbm_constraints}. Then:
\begin{enumerate}[label=(\roman*)]
    \item \textbf{Projection operator:} Both supervised and unsupervised methods implicitly learn a projection operator \(\bm{\Pi}_{\mathcal{C}_{\D}}: \mathbb{R}^d \to \mathcal{C}_{\D}\) that maps activations onto the concept cone via \(\bm{\Pi}_{\mathcal{C}_{\D}}(\a) = \hat{\bm{c}}^\top \D\) where \(\hat{\bm{c}} \geq 0\).
    \item \textbf{Shared geometric structure:} Supervised methods (CBMs) and unsupervised methods optimize for different objectives -- task performance vs. reconstruction -- but both search for a convex cone \(\mathcal{C}_{\D}\) in activation space that captures the relevant subspace of \(\SA\).
    \item \textbf{Concept recoverability:} A concept direction \(\bm{v} \in \mathbb{R}^d\) learned by one method is recoverable from dictionary \(\D\) learned by another if and only if \(\bm{v} \in \mathcal{C}_{\D}\), i.e., there exists \(\bm{\alpha} \in \mathbb{R}_+^c\) such that \(\bm{v} = \bm{\alpha}^\top \D\). Agreement between methods is thus quantified by their cone overlap.
\end{enumerate}
\end{observation}

\begin{figure*}[t]
    \vspace{-3mm}
    \centering
    \includegraphics[width=0.6\linewidth]{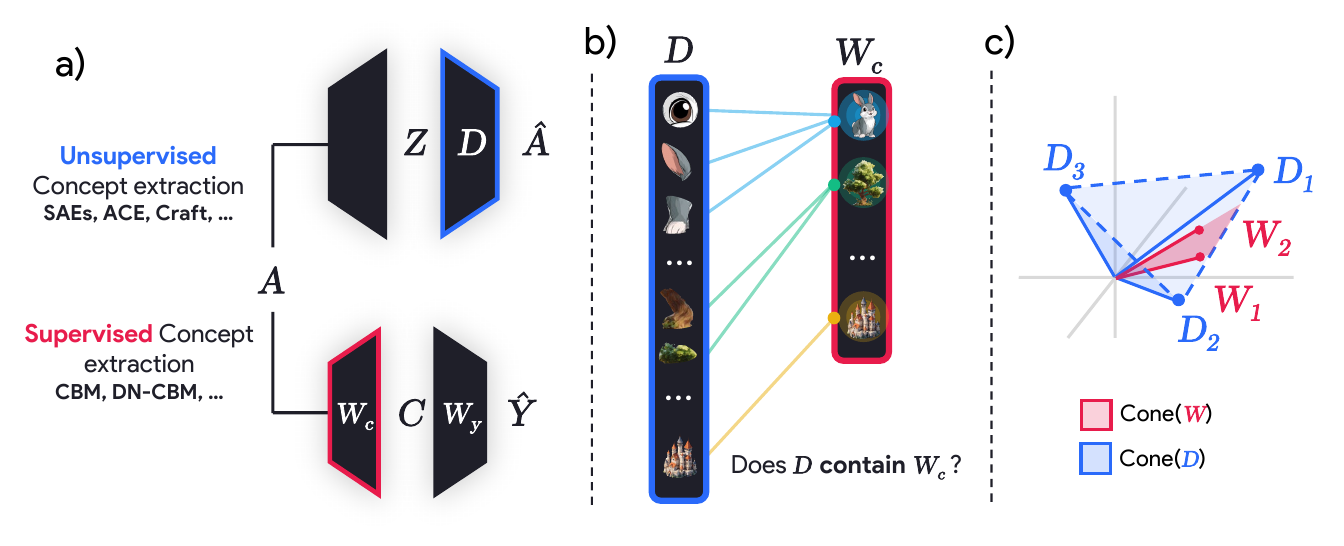}
    \vspace{-5mm}
\caption{\textbf{CBMs as Anchors for Unsupervised Concept Discovery.} While supervised CBMs learn directions $\mathbf{W}_c$ aligned with human-defined concepts, unsupervised Sparse Autoencoders (SAEs) decompose activations $\mathbf{A}$ into a dictionary $\mathbf{D}$ and sparse codes $\mathbf{Z}$. We investigate whether SAE-discovered directions contain or sparsely approximate those in CBMs. Geometrically, this tests if the supervised concept cone $\mathcal{C}_{\mathbf{W}_c} = \{\mathbf{W}_c^\top \bm{\beta} : \bm{\beta} \ge 0\}$ is contained within the unsupervised cone $\mathcal{C}_{\mathbf{D}} = \{\mathbf{D}^\top \bm{\alpha} : \bm{\alpha} \ge 0\}$. In this view, CBMs serve as partial geometric anchors. Their containment within the SAE cone quantifies whether unsupervised discovery recovers human-plausible concepts~\citep{jacovi2020towards}.
}
    \label{fig:toy_setup}
\end{figure*}

See details in Appendix~\ref{app:proof_cone_unification}.
This geometric unification motivates our approach: rather than treating supervised and unsupervised concept extraction as separate paradigms, we view them as alternative procedures for selecting a concept cone in activation space and we clarify \textit{how} CBM can actually guide unsupervised concept research: by giving some desirable concept cone that an SAE should recover.
In the next section, we operationalize this perspective by introducing metrics to quantify cone overlap and studying empirically how CBM and SAE cones align.

\subsection{Operationalization of cone overlap}
\label{sec:cone_metrics}

\textbf{Towards measuring concept recovery.} The geometric framework established in Proposition~\ref{prop:concept_cone_unification} clarifies the relationship between supervised and unsupervised concept extraction, but also raises a natural question: \emph{do SAEs actually discover plausible concepts~\citep{jacovi2020towards}, as encoded by CBMs?}
We do not expect perfect alignment -- SAEs discover a broad spectrum of concepts to reconstruct activations, while CBMs focus on a small set of task-relevant, human-desirable (or at minimum plausible \citep{jacovi2020towards}) concepts.
However, our assumption is that the SAE concept cone \(\mathcal{C}_{\D_{\text{SAE}}}\) should \emph{contain} or closely approximate the CBM concept cone \(\mathcal{C}_{\mathbf{W}_c}\).
More precisely, we expect that a \emph{sparse subset} of SAE concepts can recover CBM concepts through nonnegative combinations.
Thanks to the cone geometry, we can now formalize this intuition and derive concrete metrics to measure it. Before formally introducing our metrics, we propose to clarify exactly the role that CBM can play and why they matter for unsupervised concepts extraction research.

\textbf{Why CBM concepts matter for SAE evaluation.}
A critical challenge in SAE research is the absence of ground truth for what constitutes a "correct" concept decomposition.
Recall from Definition~\ref{def:concept} that a concept \(c_k\) is formally a coordinate of a measurable right-inverse \(\bm{H}:\SA\to\mathcal{C}\) recovering latent factors \(\bm{c}\) from activations \(\a\).
However, this inverse is generally non-unique~\citep{locatello2019challenging,klindt2025superposition}: many factorizations \(\bm{H}\) can approximately satisfy \(\bm{H}\circ\f\circ\dgp(\bm{c}) \approx \bm{c}\), and without access to the true data generating process \(\dgp\) or latent factors \(\bm{c}\), we cannot determine which is "correct."
This theoretical ambiguity manifests empirically: minor changes in hyperparameters (sparsity penalty, dictionary size, initialization) yield drastically different learned dictionaries \(\D_{\text{SAE}}\), all achieving similar reconstruction error~\citep{fel2025archetypal,paulo2025sparse}.
Recent work has proliferated diverse SAE variants -- standard SAEs~\citep{bricken2023monosemanticity}, TopK~\citep{gao2024scaling}, JumpReLU~\citep{rajamanoharan2024jumping}, Transcoders~\citep{hanna2024have}, Matching-Pursuit~\citep{costa2025flat} -- yet without ground truth, evaluation relies on proxy metrics (reconstruction loss, sparsity, task performance) that do not directly measure concept plausibility.
This is where CBMs provide crucial value: they encode a human-curated basis of concepts deemed plausible~\citep{jacovi2020towards} and desirable through explicit supervision or expert knowledge.
While CBM concepts are not ground truth (they reflect human choices, biases, and annotation limitations) they are a first order approximation of the kind of structure we seek SAEs to discover.
The concept cone framework (Definition~\ref{def:concept_cone}) operationalizes this intuition: we can quantify how well an SAE cone \(\mathcal{C}_{\D_{\text{SAE}}}\) aligns with a CBM cone \(\mathcal{C}_{\mathbf{W}_c}\) through concrete metrics.
This alignment does not constitute ground truth validation but rather a \emph{desirability signal}: an SAE that recovers human-annotated concepts through sparse, nonnegative combinations provides evidence of latching onto plausible structure, making its decompositions more likely to be meaningful, useful, and actionable.
Moreover, by comparing multiple SAE variants on their ability to recover the same CBM concepts, we can assess which inductive biases (sparsity mechanisms, expansion factors, architectures) systematically yield more plausible decompositions, providing actionable design insights.

\textbf{Formalizing the containment hypothesis.}
Having established our central hypothesis : CBMs provide human-desirable concept directions that can serve as anchors for evaluating or guiding SAEs; while recognizing their limitation as human-constructed rather than ground-truth factors, we now turn to a formal statement of this idea. The \emph{containment hypothesis} asserts that if an unsupervised dictionary faithfully captures the meaningful structure of the representation space, then the concept directions learned by a CBM should be contained within, or well-approximated by, the cone spanned by the unsupervised atoms.  
In the following, we translate this hypothesis into measurable quantities that express how much of the CBM concept space is recoverable from an SAE dictionary.
Formally, for a CBM concept direction \(\bm{v}_i \in \mathbb{R}^d\) (the \(i\)-th row of \(\mathbf{W}_c\)), we seek nonnegative coefficients \(\bm{\alpha}_i \in \mathbb{R}_+^{c_{\text{SAE}}}\) such that \(\bm{v}_i \approx \bm{\alpha}_i^\top \D_{\text{SAE}}\).
By Observation~\ref{prop:concept_cone_unification}(iii), this is equivalent to asking whether \(\bm{v}_i \in \mathcal{C}_{\D_{\text{SAE}}}\) or lies close to it.
Crucially, we further hypothesize that this recovery should be \emph{sparse}: only a small number of SAE concepts should be needed to reconstruct each CBM concept, reflecting the compositional structure of human-interpretable notions.
This leads us to formulate the following sparse nonnegative reconstruction problem for each CBM concept:
\begin{equation}
\bm{\alpha}_i^* = \argmin_{\bm{\alpha} \geq 0} \|\bm{v}_i - \bm{\alpha}^\top \D_{\text{SAE}}\|_2^2 + \lambda \|\bm{\alpha}\|_1
\label{eq:sparse_recovery}
\end{equation}
where \(\lambda > 0\) controls sparsity.
The reconstruction error \(\|\bm{v}_i - \bm{\alpha}_i^{*\top} \D_{\text{SAE}}\|_2\) measures how well the CBM concept lies in (or near) the SAE cone, while the sparsity \(\|\bm{\alpha}_i^*\|_0\) quantifies the complexity of the representation.

\subsection{Metrics}

\textbf{Measuring the geometric bridge.} The cone formalism provides a geometric bridge between supervised and unsupervised concept learning.   To assess how closely the unsupervised concept cone \(\mathcal{C}_{\D_{\text{SAE}}}\) aligns with the supervised concept cone \(\mathcal{C}_{\mathbf{W}_c}\),
we introduce a set of complementary quantitative measures that capture containment, geometric alignment, and activation-level correspondence. Together, these metrics turn the qualitative notion of ``concept alignment’’ into a concrete and measurable relation between CBMs and SAEs.

\textbf{Coverage.}
For each CBM concept direction \(\bm v_i\), i.e. the \(i\)-th row of \(\mathbf{W}_c\), we solve equation \ref{eq:sparse_recovery}. Then, we summarize cone containment with the coverage score
\begin{equation}
\mathrm{Cov}
=
\frac{
\sum_i
\|\bm{\alpha}_i^{*\top}\D_{\mathrm{SAE}}\|_2^2
}{
\sum_i
\|\bm v_i\|_2^2
}.
\label{coverage}
\end{equation}

A coverage close to one means that most CBM concept energy is captured by the SAE cone.  
A low coverage means that important CBM directions are missing from the SAE dictionary.

\textbf{Geometric alignment.}
Coverage tells us whether CBM concepts are contained in the SAE cone, but it does not tell us whether individual SAE atoms point in the same directions as CBM concepts.  
To measure axis-level similarity, we compute
\(
\rho_{\mathrm{geom}}
=
\frac{1}{c_{\mathrm{SAE}}}
\sum_j
\max_i
|\mathrm{corr}(\bm d_j,\bm w_i)|,
\)
where \(\bm d_j\) is an SAE atom and \(\bm w_i\) is a CBM direction. This metric checks whether each SAE atom is close to some CBM direction.  
However, a strict one-to-one match can be too restrictive.  
An SAE atom may correspond to a mixture of CBM concepts, and a CBM concept may be reconstructed by several SAE atoms.  
For this reason, in practice we also use a top-\(k\) version:
\begin{equation}
\rho_{\mathrm{geom}}^{(k)}
=
\frac{1}{c_{\mathrm{SAE}}}
\sum_j
\frac{1}{k}
\sum_{i\in \mathcal{N}_k(\bm d_j)}
|\mathrm{corr}(\bm d_j,\bm w_i)|,
\end{equation}
where \(\mathcal{N}_k(\bm d_j)\) is the set of the \(k\) CBM concepts most correlated with \(\bm d_j\).  
This softens the one-to-one assumption and better captures overlapping or compositional concepts.

\textbf{Activation alignment.}
Two atoms may be geometrically close but behave differently on data because the models use different encoders, thresholds, nonlinearities, or sparsity mechanisms.  
Therefore, we also compare their activation patterns over the dataset:
\(
\rho_{\mathrm{act}}
=
\frac{1}{c_{\mathrm{SAE}}}
\sum_j
\max_i
|\mathrm{corr}(\Z^{\mathrm{SAE}}_{\cdot j},
\Z^{\mathrm{CBM}}_{\cdot i})|.
\)
This measures whether SAE and CBM concepts activate on similar examples.  
As above, we also consider a top-\(k\) version to avoid forcing a strict one-to-one correspondence:
\begin{equation}
\rho_{\mathrm{act}}^{(k)}
=
\frac{1}{c_{\mathrm{SAE}}}
\sum_j
\frac{1}{k}
\sum_{i\in \mathcal{N}_k(\Z^{\mathrm{SAE}}_{\cdot j})}
|\mathrm{corr}(\Z^{\mathrm{SAE}}_{\cdot j},
\Z^{\mathrm{CBM}}_{\cdot i})|.
\end{equation}

We summarize geometric and activation alignment with
\(
\rho_{\mathrm{uni}}
=
\frac{\rho_{\mathrm{geom}}+\rho_{\mathrm{act}}}{2}.
\)

\textbf{Regression predictability.}
The previous metrics compare concepts locally, either as directions or as activation columns.  
We also measure whether the whole CBM concept space can be predicted from the SAE activations.  
We fit a linear map
\(
\hat{\Z}^{\mathrm{CBM}}
=
\Phi(\Z^{\mathrm{SAE}})
\)
and compute
\(
R^2
=
1-
\frac{
\|\Z^{\mathrm{CBM}}-\hat{\Z}^{\mathrm{CBM}}\|_F^2
}{
\|\Z^{\mathrm{CBM}}-\bar{\Z}^{\mathrm{CBM}}\|_F^2
}.
\)
%\end{equation}
A high \(R^2\) means that the information contained in the CBM concept activations can be linearly decoded from SAE activations.  
This does not imply a one-to-one concept match: the prediction may use distributed combinations of many SAE atoms.

\textbf{Sample-level agreement.}
To understand whether the SAE--CBM correspondence is concentrated or spread out, we compute
\(
p_i
=
\frac{1}{c_{\mathrm{SAE}}}
\left|\{j:m(j)=i\}\right|,
\qquad
m(j)
=
\arg\max_i
|\mathrm{corr}(\Z^{\mathrm{SAE}}_{\cdot j},
\Z^{\mathrm{CBM}}_{\cdot i})|.
\)
This shows whether many SAE atoms map to the same CBM concept or whether matches are distributed across concepts.
Then, we compare which concepts are active on each sample.  
For each sample \(t\), we compute
\(
\mathcal{T}^{\mathrm{CBM}}_t
=
\mathrm{Top}\text{-}k(\Z^{\mathrm{CBM}}_{t,\cdot}),
\qquad
\mathcal{T}^{\mathrm{SAE}}_t
=
\mathrm{Top}\text{-}k(\Z^{\mathrm{SAE}}_{t,\cdot}).
\)
We map SAE indices to CBM concepts using \(m(\cdot)\):
\(
\widehat{\mathcal{T}}^{\mathrm{SAE}\to\mathrm{CBM}}_t
=
\{m(j):j\in\mathcal{T}^{\mathrm{SAE}}_t\}.
\)
We then compute precision, recall, and F1 between
\(
\widehat{\mathcal{T}}^{\mathrm{SAE}\to\mathrm{CBM}}_t
\quad\text{and}\quad
\mathcal{T}^{\mathrm{CBM}}_t.
\)
This measures whether the most active SAE concepts correspond to the most active CBM concepts on the same input.

\textbf{Why several metrics are needed.}
No single metric is enough to compare CBMs and SAEs.  
Coverage tells us if CBM concepts are inside the SAE cone, but not if atoms are aligned.  
Geometric alignment compares directions, but not how they behave on data.  
Activation alignment compares behavior, but not whether the recovery is simple (sparse).  
Finally, \(R^2\) measures global predictability, but not one-to-one concept matching. We detail these differences in Appendix~\ref{sec:metric_theory}. In particular:
(i) high coverage does not imply high \(\rho_{\mathrm{geom}}\);
(ii) high \(\rho_{\mathrm{geom}}\) does not imply high \(\rho_{\mathrm{act}}\);
(iii) high \(\rho_{\mathrm{act}}\) does not imply sparse recovery;
(iv) high \(R^2\) does not imply one-to-one alignment.
These metrics are complementary and together provide a more complete view of SAE–CBM alignment
\section{Experiments}
\label{sec:experiments}

\textbf{Experimental settings.}
{We train three vision backbones (ResNet-50, ViT, and DINOv2) on ImageNet and CUB, and additionally a ResNet-50 on Husky vs. Wolf \citep{ribeiro2016should} to analyze classifier bias ((Appendix \ref{sec:Application}) and their role in optimizing CBM construction(Appendix \ref{sec:ApplicationV2}). }
On top of these backbones, we train six families of SAEs: 
\texttt{Vanilla}, \texttt{TopK}\citep{gao2024scaling}, \texttt{BatchTopK}\citep{bussmann2024batchtopk}, \texttt{Jump}\citep{rajamanoharan2024jumping}, \texttt{MP}\citep{costa2025flat}, and \texttt{Archetypal}\citep{fel2025archetypal}. 
Unless otherwise stated, the expansion factor and sparsity parameters are chosen following standard SAE practice, namely the expansion factor is equal to $\times 2$ and the sparsity parameters are equal to $0.005$. The value of k  for F1, $\rho_{\mathrm{act}}^{(k)}$ and $\rho_{\mathrm{geom}}^{(k)}$ is set to 20.  The CBM construction and detailed settings are given in Appendix~\ref{sec:experimentalseting}. We conduct several complementary analyses:
\nn{(i)} a sanity-check experiment validating the behavior of the metrics under controlled conditions;
\nn{(ii)} a study of how SAE variants affect geometric and statistical alignment;
\nn{(iii)} an analysis of the effects of expansion factors and sparsity constraints;
\nn{(iv)} a layer-wise exploration to study the influence of depth.%\\

%\vspace{-8mm}
\textbf{Sanity check experiment.}
The goal of this sanity check is to test whether our metrics distinguish a trained \texttt{TopK} SAE from an untrained, randomly initialized one. This is important because recent work suggests that random SAEs can obtain strong auto-interpretability scores~\citep{heap2025sparse}. Results are shown in Table~\ref{tab:sanity}.
As expected, the trained SAE achieves higher geometric and statistical alignment with the CBM. The random SAE obtains lower correlation and \(R^2\) scores, showing that its activations are less predictive of CBM concepts and capture less coherent semantic structure. The only exception is Coverage, which is higher for the random SAE. This behavior is expected: under the assumptions studied in Appendix~\ref{subsec:ConeContainement}, Proposition~\ref{prop:RecoMono} shows that a random cone can artificially inflate coverage. We think that the Coverage of the random cone is artificially inflated due to the lack of learned structure (Appendix \ref{sec:SanityV2}). 
Appendix~\ref{appendix:baselineCBM} shows that a random CBM baseline lowers all metrics, confirming the value of human-labeled concepts.
Thus, high Coverage alone is not sufficient. The four metrics must be considered together. In this case, high Coverage but low alignment and low \(R^2\) correctly reveal that the random SAE does not recover meaningful CBM-like concepts. In contrast, the trained SAE shows stronger alignment across metrics, confirming that our framework is sensitive to training quality. We observe the same trend when evaluating SAE checkpoints during training; see Appendix~\ref{Apendix:training}.

\begin{table}[t]
\centering
\scriptsize
\setlength{\tabcolsep}{3pt}
\renewcommand{\arraystretch}{1.05}

\begin{minipage}[t]{0.47\textwidth}
\centering
\caption{Sanity check results for ResNet-50, ViT-B/16 and DINOv2 on CUB and ImageNet, using the \texttt{TopK} variant with target $L_0 = 0.005$. Values are reported as mean $\pm$ standard deviation over three seeds.}
\label{tab:sanity}
\resizebox{\linewidth}{!}{
\begin{tabular}{c|lcccc}
\toprule
~ & \textbf{Dataset / Method}
& $\mathrm{Coverage}$
& $R^2$
& $\mathrm{Top}\text{-}k~F_1$
& $\rho_{\mathrm{uni}}$ \\
\midrule

\multirow{4}{*}{\rotatebox{90}{R50}}
& CUB / normal
& $0.2483 \pm 0.0184$
& $\mathbf{0.9566 \pm 0.0046}$
& $\mathbf{0.0920 \pm 0.0026}$
& $\mathbf{0.0829 \pm 0.0021}$ \\

& CUB / random
& $\mathbf{0.6500 \pm 0.0356}$
& $0.8503 \pm 0.0099$
& $0.0599 \pm 0.0046$
& $0.0600 \pm 0.0006$ \\

\cmidrule(l){2-6}

& IN / normal
& $0.1718 \pm 0.0044$
& $\mathbf{0.7090 \pm 0.0027}$
& $\mathbf{0.1371 \pm 0.0001}$
& $\mathbf{0.2078 \pm 0.0019}$ \\

& IN / random
& $\mathbf{0.8683 \pm 0.0067}$
& $0.3184 \pm 0.0020$
& $0.0378 \pm 0.0013$
& $0.1011 \pm 0.0006$ \\

\midrule

\multirow{4}{*}{\rotatebox{90}{ViT}}
& CUB / normal
& $0.1953 \pm 0.0046$
& $\mathbf{0.8070 \pm 0.0126}$
& $\mathbf{0.1292 \pm 0.0048}$
& $\mathbf{0.0922 \pm 0.0019}$ \\

& CUB / random
& $\mathbf{0.7498 \pm 0.0000}$
& $0.7067 \pm 0.0118$
& $0.0726 \pm 0.0026$
& $0.0919 \pm 0.0004$ \\

\cmidrule(l){2-6}

& IN / normal
& $0.0127 \pm 0.0003$
& $\mathbf{0.7630 \pm 0.0009}$
& $\mathbf{0.1719 \pm 0.0003}$
& $\mathbf{0.3428 \pm 0.0010}$ \\

& IN / random
& $\mathbf{0.9140 \pm 0.0249}$
& $0.2043 \pm 0.0013$
& $0.0473 \pm 0.0004$
& $0.1340 \pm 0.0003$ \\

\midrule

\multirow{4}{*}{\rotatebox{90}{DINOv2}}
& CUB / normal
& $0.1393 \pm 0.0223$
& $\mathbf{0.8597 \pm 0.0208}$
& $\mathbf{0.1242 \pm 0.0031}$
& $\mathbf{0.1269 \pm 0.0059}$ \\

& CUB / random
& $\mathbf{0.7369 \pm 0.0333}$
& $0.7706 \pm 0.0223$
& $0.0902 \pm 0.0029$
& $0.1022 \pm 0.0010$ \\

\cmidrule(l){2-6}

& IN / normal
& $0.0098 \pm 0.0000$
& $\mathbf{0.7670 \pm 0.0006}$
& $\mathbf{0.1511 \pm 0.0006}$
& $\mathbf{0.3532 \pm 0.0006}$ \\

& IN / random
& $\mathbf{0.8678 \pm 0.0176}$
& $0.2670 \pm 0.0021$
& $0.0477 \pm 0.0012$
& $0.1643 \pm 0.0017$ \\

\bottomrule
\end{tabular}
}
\end{minipage}
\hfill
\begin{minipage}[t]{0.47\textwidth}
\centering
\caption{Comparison of different SAE variants on CUB using ResNet-50, ViT and DINOv2. Values are reported as mean $\pm$ standard deviation over three seeds.}
\label{tab:sae_comparison}
\resizebox{\linewidth}{!}{
\begin{tabular}{c|lcccc}
\toprule
\textbf{Dataset-Backbone} 
& \textbf{SAE Type} 
& \textbf{Coverage} 
& $R^2$ 
& $\mathrm{Top}\text{-}k~F_1$ 
& $\rho_{\mathrm{uni}}$ \\ 
\midrule

\multirow{6}{*}{\rotatebox{90}{CUB--R50}} 
& \texttt{Vanilla}    
& $0.7292 \pm 0.0467$ 
& $0.9775 \pm 0.0059$ 
& $0.0733 \pm 0.0011$ 
& $\mathbf{0.1081 \pm 0.0019}$ \\

& \texttt{TopK}       
& $0.2548 \pm 0.0450$ 
& $0.9559 \pm 0.0085$ 
& $0.0926 \pm 0.0007$ 
& $0.0828 \pm 0.0048$ \\

& \texttt{BatchTopK}  
& $0.6412 \pm 0.0190$ 
& $0.9616 \pm 0.0081$ 
& $0.0925 \pm 0.0068$ 
& $0.0868 \pm 0.0042$ \\

& \texttt{Jump}       
& $\mathbf{0.7320 \pm 0.0427}$ 
& $\mathbf{0.9778 \pm 0.0061}$ 
& $0.0730 \pm 0.0024$ 
& $0.1046 \pm 0.0018$ \\

& \texttt{MP}         
& $0.6481 \pm 0.0533$ 
& $0.6859 \pm 0.0408$ 
& $0.0543 \pm 0.0035$ 
& $0.0350 \pm 0.0002$ \\

& \texttt{Archetypal} 
& $0.5410 \pm 0.0265$ 
& $0.9553 \pm 0.0077$ 
& $\mathbf{0.0987 \pm 0.0006}$ 
& $0.0915 \pm 0.0058$ \\

\midrule

\multirow{6}{*}{\rotatebox{90}{CUB--ViT}} 
& \texttt{Vanilla}    
& $0.5758 \pm 0.0208$ 
& $\mathbf{0.8985 \pm 0.0163}$ 
& $0.1015 \pm 0.0043$ 
& $\mathbf{0.1324 \pm 0.0037}$ \\

& \texttt{TopK}       
& $0.1905 \pm 0.0181$ 
& $0.7983 \pm 0.0182$ 
& $0.1290 \pm 0.0024$ 
& $0.0901 \pm 0.0060$ \\

& \texttt{BatchTopK}  
& $0.2790 \pm 0.0029$ 
& $0.8329 \pm 0.0152$ 
& $0.1274 \pm 0.0022$ 
& $0.0963 \pm 0.0056$ \\

& \texttt{Jump}       
& $0.5767 \pm 0.0218$ 
& $0.8969 \pm 0.0154$ 
& $0.0989 \pm 0.0070$ 
& $0.1299 \pm 0.0037$ \\

& \texttt{MP}         
& $\mathbf{0.5954 \pm 0.0299}$ 
& $0.7740 \pm 0.0246$ 
& $\mathbf{0.1328 \pm 0.0024}$ 
& $0.0706 \pm 0.0030$ \\

& \texttt{Archetypal} 
& $0.2740 \pm 0.0029$ 
& $0.8369 \pm 0.0141$ 
& $0.1300 \pm 0.0020$ 
& $0.1021 \pm 0.0061$ \\

\midrule

\multirow{6}{*}{\rotatebox{90}{CUB--DINOv2}} 
& \texttt{Vanilla}    
& $0.7498 \pm 0.0275$ 
& $\mathbf{0.9431 \pm 0.0130}$ 
& $0.1067 \pm 0.0013$ 
& $\mathbf{0.1402 \pm 0.0015}$ \\

& \texttt{TopK}       
& $0.1407 \pm 0.0207$ 
& $0.8590 \pm 0.0208$ 
& $0.1237 \pm 0.0032$ 
& $0.1269 \pm 0.0061$ \\

& \texttt{BatchTopK}  
& $0.3773 \pm 0.0191$ 
& $0.8726 \pm 0.0195$ 
& $0.1234 \pm 0.0021$ 
& $0.1114 \pm 0.0042$ \\

& \texttt{Jump}       
& $\mathbf{0.7518 \pm 0.0294}$ 
& $0.9408 \pm 0.0136$ 
& $0.1078 \pm 0.0015$ 
& $0.1385 \pm 0.0017$ \\

& \texttt{MP}         
& $0.6306 \pm 0.0089$ 
& $0.8317 \pm 0.0258$ 
& $0.1207 \pm 0.0014$ 
& $0.0838 \pm 0.0019$ \\

& \texttt{Archetypal} 
& $0.3687 \pm 0.0152$ 
& $0.8744 \pm 0.0199$ 
& $\mathbf{0.1249 \pm 0.0024}$ 
& $0.1169 \pm 0.0044$ \\

\bottomrule
\end{tabular}
}
\end{minipage}

\vspace{-2mm}
\end{table}

\vspace{-2mm}
\textbf{Jump and vanilla SAEs show superior alignment.}
We compare several SAE variants on CUB and ImageNet using ResNet-50, ViT, and DINOv2 backbones. Results are reported in Table~\ref{tab:sae_comparison}, with comparable sparsity levels when feasible. The main experiments use a Joint CBM~\cite{koh2020concept} as the reference; in Appendix~\ref{Appendix:indjointseq}, we show that using other CBM variants leads to similar conclusions.
All variants achieve high $R^2$ values, indicating substantial predictive alignment with CBM representations. However, \texttt{Jump} and \texttt{Vanilla} consistently exhibit superior coverage alongside competitive geometric alignment, suggesting their regularization schemes facilitate discovery of broader concept repertoires without sacrificing directional correspondence. \texttt{Archetypal} provides better F1 results. 

While no method dominates uniformly across all metrics, these results indicate that sparsity mechanism constitutes a consequential architectural choice shaping both the geometry and plausibility of learned concept spaces.

\begin{figure*}[t]
\vspace{-3mm}
\begin{center}
\begin{tabular}{ccc}
  
 \includegraphics[width=0.25\linewidth]{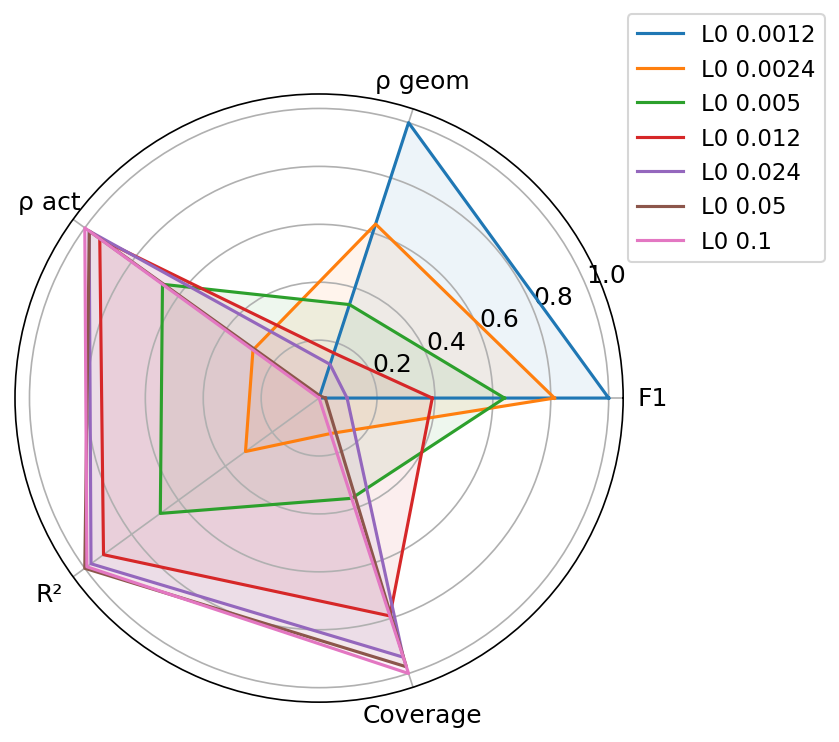} & 
  \includegraphics[width=0.25\linewidth]{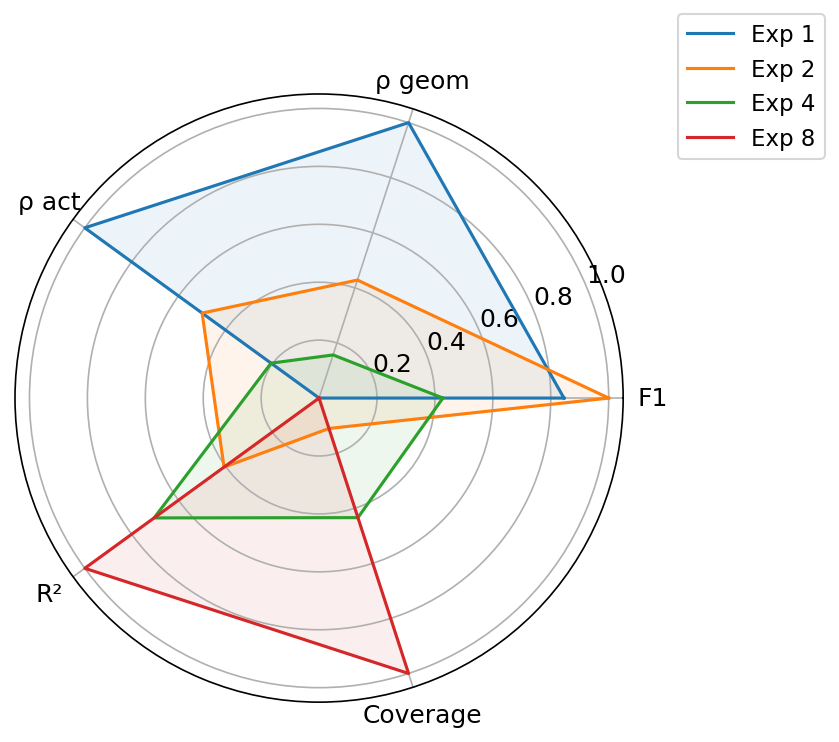}    &
  \includegraphics[width=0.25\linewidth]{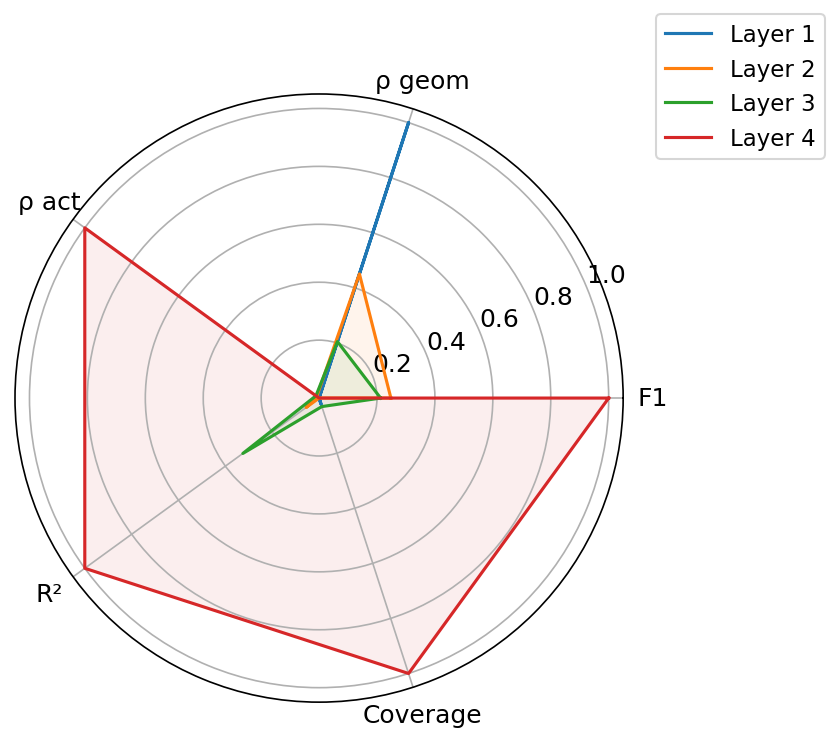} \\
  (a) Sparsity target   &   (b) Expansion factor &   (c) Features level 
\end{tabular}
\end{center}
\caption{
Radar plot illustrating the influence of the \textbf{(a) sparsity target }, \textbf{(b) expansion factor}, and \textbf{(c) features level}  on \texttt{TopK} Sparse Autoencoder (SAE) performance for both CUB and ImageNet datasets. 
Each axis corresponds to a normalized metric. 
}
\label{fig:results}
\end{figure*}
\vspace{-2mm}
\textbf{A sparsity sweet spot.}
We systematically vary sparsity in \texttt{TopK} SAEs via the target $\ell_0$ parameter, which controls the proportion of active units per sample. Results across CUB and ImageNet with multiple backbones (Appendix \ref{sec:experimentsSparsity}, ~\cref{fig:results}(a)) reveal a consistent pattern: as sparsity decreases—permitting more active dictionary atoms—coverage, $R^2$, and F1 improve substantially, indicating that denser codes better reconstruct the CBM concept space. Geometric alignment ($\rho_{\text{geom}}$), however, exhibits the opposite trend, declining as representations become less concentrated. This inverse relationship suggests a potential tension: aggressive sparsity produces plausible, localized prototypes at the expense of completeness, while relaxed sparsity enhances concept recovery but disperses structure across more atoms. Empirically, intermediate sparsity regimes ($\approx$ 0.01–0.05\% of sparsity) achieve a favorable balance, maintaining reasonable geometric fidelity while attaining high coverage. Once again, these findings indicate that sparsity is a substantive design parameter governing the plausibility of the found dictionary.

\textbf{Expansion factor sweet spot.}
We examine the influence of latent expansion factor (the ratio of dictionary size to input dimensionality) on \texttt{TopK} SAE performance. Increasing this factor enlarges representational capacity, permitting the model to discover more distinct concept directions. Results across CUB and ImageNet (Appendix \ref{sec:experimentsExpansionfactor}, ~\cref{fig:results}(b)) demonstrate that larger latent spaces systematically improve coverage, $R^2$, though geometric alignment ($\rho_{\text{geom}}$), activation correlation ($\rho_{\text{act}}$) and F1 decline modestly as representations distribute across more atoms. A sweet spot emerges at approximately $2\times$ expansion.

Beyond this threshold, further expansion yields diminishing returns: coverage saturates while directional alignment continues to diffuse. These findings suggest that expansion factor, like sparsity, governs a tradeoff between expressive capacity and interpretive concentration, with intermediate regimes offering the most favorable balance for recovering plausible concepts.

\vspace{-2mm}
\textbf{Layer-wise evaluation of \texttt{topK} SAEs.}
We investigate how the semantic alignment between the SAE and the CBM evolves across different stages of the ResNet-50 backbone. Specifically, we extract features \textit{after the pooling operation of each residual block} (i.e., after the spatial average pooling of stages 1–4) and train a separate \texttt{TopK} SAE on each set of features. This setup allows us to assess at which depth the backbone’s representation best matches the CBM concept space. As shown in Appendix \ref{sec:experimentsLayer_wise} and ~\cref{fig:results}(c), lower stages (Layers 1–2) capture mostly low-level visual attributes such as edges, colors, and textures, while deeper stages (Layers 3–4) encode increasingly abstract and semantically structured information.
Results reveal a clear depth-dependent trend. On CUB, deeper layers (particularly Layers 3 and 4) yield the strongest alignment, as indicated by higher $R^2$, coverage, and F1 values. Layer 4 achieves the best results in terms of predictability, strong concept coverage, and higher F1. Conversely, earlier layers, despite showing higher geometric correlation, primarily represent local patterns that do not align with the CBM’s semantic dimensions.  For ImageNet, the same tendency holds. Layer 4 again achieves the highest $R^2$ and coverage, confirming that post-pooling representations in the later blocks of ResNet-50 encode higher-level, concept-like features consistent with CBM semantics.

\vspace{-2mm}
\section{Conclusion}
\vspace{-2mm}
This paper introduced concept cones as a geometric framework for unifying supervised and unsupervised concept learning. By showing that both Concept Bottleneck Models and Sparse Autoencoders can be viewed as learning dictionaries whose nonnegative combinations span a cone in activation space, we provide a common language for comparing human-supervised and data-driven concept representations. Within this framework, CBMs act as partial semantic anchors: they do not define a complete ground truth for interpretability, but they provide a useful reference for testing whether unsupervised SAE dictionaries recover human-plausible directions.
Our experiments show that this geometric view leads to practical evaluation tools. The proposed metrics—coverage, geometric alignment, activation alignment, regression predictability, and sample-level agreement—capture complementary aspects of SAE–CBM correspondence, and the results confirm that no single metric is sufficient.
Our experiments reveal that intermediate layers of pretrained backbones, particularly those following block-level pooling in ResNet architectures, exhibit the strongest semantic alignment with concept-level representations. Moderate sparsity and expanded latent capacity enable SAEs to uncover structured, semantically meaningful concept dictionaries that parallel CBM-defined semantic dimensions.  Together, these findings demonstrate that interpretability can emerge from the interplay between supervision and sparsity, offering both control and scalability. By establishing a unified framework for evaluating semantic alignment across supervised and unsupervised settings, this work provides a foundation for a more systematic, operational understanding of how deep networks internalize human-aligned structure. Ultimately, our results suggest that the path toward truly interpretable foundation models lies not in choosing between supervision or discovery, but in harmonizing the two within a common geometric and statistical framework.

\paragraph{Limitations.} This work has limitations. Firstly, CBMs should not be interpreted as complete ground-truth representations. Human-provided concepts are biased, depending on the concept vocabulary and the limited annotations and the alignment with a CBM should be seen as an agreement with a particular semantic reference. Secondly, the proposed metrics are diagnostic and should not be taken as standalone metrics. Results should therefore be interpreted through the behaviors of all the metrics. 

\paragraph{Impact statement.}
This work aims to advance explainable AI and the understanding of deep neural networks by linking SAEs and CBMs. There are many potential societal consequences of our work, none of which we feel must be specifically highlighted here.
\clearpage

\textbf{Acknowledgments:} 
The authors thank Demba Ba for many fruitful discussions. 
This work was granted access to the HPC resources of IDRIS under the allocation 2024-AD011011970R4. 
This work has been made possible in part by a gift from the Chan Zuckerberg Initiative Foundation to establish the Kempner Institute for the Study of Natural and Artificial Intelligence at Harvard University. T.F. is supported by the Kempner Institute Research Fellowship.

\bibliographystyle{plainnat}
\bibliography{main}

%%%%%%%%%%%%%%%%%%%%%%%%%%%%%%%%%%%%%%%%%%%%%%%%%%%%%%%%%%%%%%%%%%%%%%%%%%%%%%%
%%%%%%%%%%%%%%%%%%%%%%%%%%%%%%%%%%%%%%%%%%%%%%%%%%%%%%%%%%%%%%%%%%%%%%%%%%%%%%%
% APPENDIX
%%%%%%%%%%%%%%%%%%%%%%%%%%%%%%%%%%%%%%%%%%%%%%%%%%%%%%%%%%%%%%%%%%%%%%%%%%%%%%%
%%%%%%%%%%%%%%%%%%%%%%%%%%%%%%%%%%%%%%%%%%%%%%%%%%%%%%%%%%%%%%%%%%%%%%%%%%%%%%%
\newpage
\appendix
\onecolumn

\clearpage
\setcounter{page}{1}

\begin{center}
    {\huge \textbf{A Geometric Unification of Concept Learning with Concept Cones}} \\  {\Large Supplementary Material}
\end{center}

\renewcommand{\thetable}{A.\arabic{table}}
\renewcommand{\thefigure}{A.\arabic{figure}}

\section{Detailed CBM variant descriptions}
\label{app:cbm_variants}

In this section, we provide comprehensive descriptions of each CBM variant presented in Equation~\ref{eq:cbm_constraints}, clarifying their specific architectures, training procedures, and key distinctions.

\subsection{Joint and sequential CBM}
\label{app:cbm_joint_seq}

\paragraph{Joint CBM \citep{koh2020concept}.}
The original CBM formulation where both the concept prediction layer \(\mathbf{W}_c\) and the task prediction layer \(\mathbf{W}_y\) are trained simultaneously with \(\lambda > 0\).
The model minimizes both concept prediction loss (via binary cross-entropy on labeled concepts \(\bm{c}_i\)) and task prediction loss (via cross-entropy on labels \(\bm{y}_i\)) jointly.
This requires a dataset with both concept annotations and task labels.
The joint training allows the model to learn concepts that are both semantically meaningful and maximally useful for the downstream task, but may lead to concepts that are less faithful to their intended human interpretation. 

\paragraph{Sequential CBM \citep{koh2020concept}.}
A two-stage training procedure where the concept layer \(\mathbf{W}_c\) is first trained to predict concept labels (stage 1), then frozen, and the task prediction layer \(\mathbf{W}_y\) is trained on top of the predicted concepts (stage 2).
This approach enforces a stricter bottleneck since the task predictor can only access information through the learned concepts, potentially improving interpretability at the cost of some task performance.
Sequential training ensures concepts remain faithful to their annotations, as they are optimized solely for concept prediction in the first stage.

\paragraph{Independent CBM \citep{koh2020concept}.}
Concept and task predictors are trained completely independently: \(\mathbf{W}_c\) is trained to predict concepts, and \(\mathbf{W}_y\) is trained directly on the original features \(\a_i\) to predict tasks.
At inference time, the independently trained components are composed: predictions flow through \(\mathbf{W}_c\) to produce concepts, which are then fed to \(\mathbf{W}_y\).
This variant is primarily used for analysis and comparison rather than practical deployment.

\subsection{Label-Free CBM (LF-CBM)}
\label{app:cbm_lf}

\paragraph{Method \citep{oikarinen2023label}.}
Label-Free CBM eliminates the need for concept annotations (\(\lambda = 0\)) by leveraging large language models (LLMs) to generate task-relevant concepts.
Given class names, GPT-3 is prompted to generate a set of discriminative concepts for each class (e.g., "red feathers," "curved beak" for bird species).
These textual concepts are embedded using CLIP's text encoder to obtain \(\D_{\texttt{LLM}} \in \mathbb{R}^{c \times d}\).
The concept prediction layer is then defined as \(\mathbf{W}_c = \D_{\texttt{LLM}}\), and the model learns to align image features to these text-derived concept directions.

\paragraph{Architecture details.}

For an input image \(\x_i\), the model computes:
\begin{align}
\a_i &= \f(\x_i) \quad \text{(CLIP image features)} \\
\hat{\bm{c}}_i &= \sigma(\a_i \D_{\texttt{LLM}}^\top + \bm{b}_c) \quad \text{(concept predictions)} \\
\hat{\bm{y}}_i &= \mathrm{softmax}(\hat{\bm{c}}_i \mathbf{W}_y^\top + \bm{b}_y) \quad \text{(class predictions)}
\end{align}

The model is trained end-to-end with only task labels, learning to activate the "correct" LLM-generated concepts for classification.

\paragraph{Limitations.}
There is no guarantee that LLM-generated concepts are actually detectable in the learned features \(\f\), potentially leading to unfaithful explanations where concept activations do not correspond to their semantic meaning \citep{margeloiu2021concept,roth2023waffling}.

\subsection{Post-hoc CBM}
\label{app:cbm_posthoc}

\paragraph{Method \citep{yuksekgonul2022post}.}
Post-hoc CBM constructs concept bottlenecks from already-trained neural networks without retraining.
Given a trained feature extractor \(\f\) and a set of labeled concept examples, the method learns a linear concept bank \(\D_{\texttt{Bank}}\) that decomposes feature activations into interpretable directions.

\paragraph{Concept bank construction.}
Several approaches can be used:
\begin{itemize}
    \item \textbf{Supervised projection}: For each concept, collect positive and negative examples, compute their mean features, and define the concept direction as the normalized difference.
    \item \textbf{Matrix factorization}: Apply PCA, NMF, or ICA to activation matrices to discover latent directions, then label them post-hoc.
    \item \textbf{Probe training}: Train linear probes on features to predict binary concept labels, using probe weights as concept directions.
\end{itemize}

\paragraph{Architecture with residual.}
A key variant of Post-hoc CBM includes a residual connection to preserve information:
\begin{align}
\hat{\bm{c}}_i &= \sigma(\a_i \D_{\texttt{Bank}}^\top + \bm{b}_c) \\
\hat{\bm{y}}_i &= \mathrm{softmax}([\hat{\bm{c}}_i \;|\; \a_i] \mathbf{W}_y^\top + \bm{b}_y)
\end{align}
where \([\hat{\bm{c}}_i \;|\; \a_i]\) denotes concatenation.
The residual connection allows the model to maintain original task performance while providing concept-based explanations, at the cost of reduced faithfulness to the bottleneck constraint.

\subsection{DN-CBM (Discover-then-Name CBM)}
\label{app:cbm_dn}

\paragraph{Method \citep{rao2024discover}.}
DN-CBM inverts the traditional CBM paradigm by first discovering concepts the model has learned, then naming them, and finally using them for classification.
The method uses Sparse Autoencoders (SAEs) to decompose CLIP features into interpretable concept directions in an unsupervised, task-agnostic manner.

\paragraph{Concept discovery via SAE.}
An SAE is trained to reconstruct CLIP features with sparsity constraints:
\begin{align}
\mathcal{L}_{\text{SAE}}(\a) = \|\text{SAE}(\a) - \a\|_2^2 + \lambda_1 \|\phi(f(\a))\|_1
\end{align}
where \(f(\a) = \mathbf{W}_E^\top \a\) is the encoder, \(\phi\) is ReLU, and \(g(\cdot) = \mathbf{W}_D^\top \phi(f(\a))\) is the decoder.
The dictionary \(\D_{\texttt{SAE}} = \mathbf{W}_D \in \mathbb{R}^{c \times d}\) contains learned concept directions as rows.

\paragraph{Automated concept naming.}
Each SAE neuron (dictionary atom) is assigned a name by finding the closest text embedding in CLIP space:
\begin{align}
s_c = \argmax_{v \in \mathcal{V}} \cos(\vec{p}_c, \text{CLIP}_{\text{text}}(v))
\end{align}
where \(\vec{p}_c = [\mathbf{W}_D]_c\) is the dictionary vector and \(\mathcal{V}\) is a vocabulary (e.g., 20k common English words).

\paragraph{Task-agnostic bottleneck.}
The SAE is trained once on a large unlabeled dataset (e.g., CC3M), then frozen and used as a concept bottleneck for multiple downstream tasks:
\begin{align}
\hat{\bm{c}}_i &= \phi(\a_i \D_{\texttt{SAE}}^\top) \quad \text{(sparse concept activations)} \\
\hat{\bm{y}}_i &= \mathrm{softmax}(\hat{\bm{c}}_i \mathbf{W}_y^\top + \bm{b}_y) \quad \text{(task-specific linear classifier)}
\end{align}

Only the linear classifier \(\mathbf{W}_y\) is trained per task, making DN-CBM highly efficient and scalable.

\subsection{SpLiCE (Sparse Linear Concept Embeddings)}
\label{app:cbm_splice}

\paragraph{Method \citep{bhalla2024interpreting}.}
SpLiCE interprets CLIP models by decomposing image features into sparse linear combinations of predefined text concept embeddings.
Unlike DN-CBM which learns the dictionary, SpLiCE uses a fixed dictionary \(\D_{\texttt{CLIP}}^\top\) consisting of CLIP text embeddings of a concept vocabulary.

\paragraph{Architecture.}
For each image, SpLiCE solves an optimization problem to find sparse concept coefficients:
\begin{align}
\hat{\bm{c}}_i = \argmin_{\bm{c}} \|\a_i - \D_{\texttt{CLIP}}^\top \bm{c}\|_2^2 + \lambda \|\bm{c}\|_1
\end{align}
where \(\a_i = \text{CLIP}_{\text{image}}(\x_i)\) and \(\D_{\texttt{CLIP}}^\top \in \mathbb{R}^{d \times c}\) contains CLIP text embeddings as columns.

\paragraph{Key distinction.}
SpLiCE is primarily an \emph{interpretation} method rather than a classification model.
The sparse codes \(\hat{\bm{c}}_i\) explain what concepts CLIP "sees" in an image, but are not used to train a downstream classifier.
The method is post-hoc and does not require retraining the feature extractor.

\subsection{Stochastic CBM}
\label{app:cbm_stochastic}

\paragraph{Method \citep{vandenhirtz2024stochastic}.}
Stochastic CBM (also called Probabilistic CBM \citep{kim2023probabilistic}) replaces deterministic concept predictions with distributions to capture uncertainty.

\paragraph{Architecture.}
Instead of directly predicting concept activations, the model learns parameters of a Gaussian distribution:
\begin{align}
\bm{\mu}(\a_i) &= \mathbf{W}_{\mu}^\top \a_i + \bm{b}_{\mu} \\
\bm{\Sigma}(\a_i) &= \text{softplus}(\mathbf{W}_{\Sigma}^\top \a_i + \bm{b}_{\Sigma}) \\
\bm{\eta}_i &\sim \mathcal{N}(\bm{\mu}(\a_i), \text{diag}(\bm{\Sigma}(\a_i))) \\
\hat{\bm{c}}_i &= \sigma(\bm{\eta}_i)
\end{align}

During training, samples are drawn from \(\mathcal{N}(\bm{\mu}, \bm{\Sigma})\) and passed through the sigmoid to produce concept probabilities.
At inference, the mean \(\bm{\mu}(\a_i)\) can be used, or multiple samples can be drawn to quantify prediction uncertainty.

\paragraph{Concept directions.}
The concept directions are encoded in \(\mathbf{W}_{\mu}\), which can be viewed as \(\mathbf{W}_c\) in the deterministic case.
The stochasticity allows the model to express when it is uncertain about a concept's presence, improving calibration and intervention efficacy.

\subsection{CLIP-QDA}
\label{app:cbm_clipqda}

\paragraph{Method \citep{kazmierczak2023clip}.}
CLIP-QDA constructs an interpretable concept bottleneck using Quadratic Discriminant Analysis (QDA) in a concept space derived from CLIP text embeddings.

\paragraph{Concept space construction.}
\begin{enumerate}
    \item Start with a vocabulary \(\mathcal{V}\) of concept words (e.g., 20k common English words)
    \item Embed each word using CLIP text encoder: \(t_i = \text{CLIP}_{\text{text}}(v_i)\)
    \item Apply PCA to reduce dimensionality: \(\D_{\texttt{CLIP}} = \text{PCA}(\{t_1, \ldots, t_{|\mathcal{V}|}\}) \in \mathbb{R}^{c \times d}\)
    \item For an image, compute concept representation: \(\hat{\bm{c}}_i = \a_i^\top \D_{\texttt{CLIP}}^\top\)
\end{enumerate}

\paragraph{Quadratic discriminant analysis.}
Instead of learning a linear classifier \(\mathbf{W}_y\), CLIP-QDA learns class-specific Gaussian distributions:
\begin{align}
p(\hat{\bm{c}} \mid y) &= \mathcal{N}(\bm{\mu}_y, \bm{\Sigma}_y) \\
\hat{y}_i &= \argmax_{y} \left[ \log p(\hat{\bm{c}}_i \mid y) + \log p(y) \right]
\end{align}

This results in quadratic decision boundaries in concept space, which can model more complex class distributions than linear CBMs.

\section{Concept cones: a unifying geometric framework.}
\label{app:proof_cone_unification}

\textbf{Concept recovery formulation.}\begin{definition}[Concept as inversion of latent factors]
\label{def:concept}
A concept $c_k$ is a coordinate of a measurable right-inverse \(\bm{H}:\SA\to\mathcal{C}\) of the data-to-feature map \(\f\circ\dgp\).
Formally, \(\bm{H}\circ\f\circ\dgp(\bm{c}) \approx \bm{c}\) in expectation under \(P_{\bm{c}}\).
Equivalently, the \(k\)-th concept is \(\bm{H}_k(\a)\), which estimates the latent factor \(c_k\) from an activation \(\a\in\SA\).
\end{definition}

With Definition~\ref{def:concept} in hand, concept recovery is the problem of inverting the map \(\f\circ\dgp\) on its image. 
Concretely, one seeks a measurable right-inverse \(\bm{H}:\SA\to\mathcal{C}\) so that \(\bm{H}\circ\f\circ\dgp \approx \mathrm{Id}_{\mathcal{C}}\) in expectation. 
Because \(\f\) is generally nonlinear and high-dimensional, we work in a linear chart of \(\SA\) and approximate inversion by learning a low-dimensional coordinate system for activations. 
This leads naturally to concept extraction as \emph{dictionary learning} formulation~\citep{klindt2023identifying,klindt2020towards,fel2023holistic} that we recall now.

\paragraph{Projection, geometry, and recoverability in concept cones.}We provide a detailed proof of Proposition~\ref{prop:concept_cone_unification}, which establishes that concept cones provide a unifying geometric framework for both supervised (CBM) and unsupervised (SAE, CRAFT, etc.) concept extraction methods.

We recall our setup with our activation space: \(\SA \subset \mathbb{R}^d\), with activations \(\a_i = \f(\x_i)\), the concept dictionary: \(\D \in \mathbb{R}^{c \times d}\) (rows are concept directions) and the associated concept codes: \(\bm{z}_i \in \mathbb{R}^c\) (unsupervised) or \(\hat{\bm{c}}_i \in [0,1]^c\) (supervised). We denote the concept cone: \(\mathcal{C}_{\D} = \{\bm{v} \in \mathbb{R}^d : \bm{v} = \bm{\alpha}^\top \D, \bm{\alpha} \in \mathbb{R}_+^c\}\).
For unsupervised methods (Equation~\ref{eq:dico_constraints}), the objective is:
\begin{equation}
\min_{\bm{Z}, \D} \|\A - \bm{Z}\D\|_F^2 + \Omega(\bm{Z}, \D)
\label{eq:unsup_obj}
\end{equation}
where \(\Omega(\cdot)\) encodes constraints (sparsity, nonnegativity, orthogonality, etc.), see \cref{eq:dico_constraints} or details.

For supervised methods (\cref{eq:cbm_constraints}), the objective is:
\begin{equation}
\min_{\mathbf{W}_c, \mathbf{W}_y} \frac{1}{n}\sum_{i=1}^n \mathcal{L}_{\mathrm{CE}}(\hat{\bm{y}}_i, \bm{y}_i) + \lambda \mathcal{L}_{\mathrm{BCE}}(\hat{\bm{c}}_i, \bm{c}_i)
\label{eq:sup_obj}
\end{equation}
where \(\hat{\bm{c}}_i = \sigma(\a_i \mathbf{W}_c^\top + \bm{b}_c)\) and \(\hat{\bm{y}}_i = \mathrm{softmax}(\hat{\bm{c}}_i \mathbf{W}_y^\top + \bm{b}_y)\). See \cref{eq:cbm_constraints} and \cref{app:cbm_variants} for details. We recall our claim: 

\begin{observation}[Concept Cones Unify Extraction Methods]
Let \(\D \in \mathbb{R}^{c \times d}\) be a concept dictionary learned by any method in ~\cref{eq:dico_constraints} or \cref{eq:cbm_constraints}. Then:
\begin{enumerate}[label=(\roman*)]
    \item \textbf{Projection operator:} Both supervised and unsupervised methods implicitly learn a projection operator \(\bm{\Pi}_{\mathcal{C}_{\D}}: \mathbb{R}^d \to \mathcal{C}_{\D}\) that maps activations onto the concept cone via \(\bm{\Pi}_{\mathcal{C}_{\D}}(\a) = \hat{\bm{c}}^\top \D\) where \(\hat{\bm{c}} \geq 0\).
    \item \textbf{Shared geometric structure:} Supervised methods (CBMs) and unsupervised methods optimize for different objectives—task performance vs. reconstruction—but both search for a convex cone \(\mathcal{C}_{\D}\) in activation space that captures the relevant subspace of \(\SA\).
    \item \textbf{Concept recoverability:} A concept direction \(\bm{v} \in \mathbb{R}^d\) learned by one method is recoverable from dictionary \(\D\) learned by another if and only if \(\bm{v} \in \mathcal{C}_{\D}\), i.e., there exists \(\bm{\alpha} \in \mathbb{R}_+^c\) such that \(\bm{v} = \bm{\alpha}^\top \D\). Agreement between methods is thus quantified by their cone overlap.
\end{enumerate}
\end{observation}
\begin{proof}

\textbf{(i): Projection operators}

We first show that, for any unsupervised method in Equation~\ref{eq:dico_constraints} that enforces \(\bm{Z} \geq 0\) (e.g., SAE with ReLU, CRAFT with nonnegativity, NMF), the reconstruction \(\hat{\a}_i = \bm{z}_i^\top \D\) lies in the concept cone \(\mathcal{C}_{\D}\).
    
By definition, \(\hat{\a}_i = \bm{z}_i^\top \D = \sum_{j=1}^c z_{ij} \D_j\) where \(\D_j\) is the \(j\)-th row of \(\D\).
Since the constraint ensures \(\bm{z}_i \geq 0\) (either explicitly via \(\bm{Z} \geq 0\) in CRAFT/NMF, or implicitly via \(\phi(\cdot) = \mathrm{ReLU}(\cdot)\) in SAEs), we have \(z_{ij} \geq 0\) for all \(j\).
Therefore, \(\hat{\a}_i\) is a nonnegative linear combination of the rows of \(\D\), which by definition means \(\hat{\a}_i \in \mathcal{C}_{\D}\).
Thus, the mapping \(\a_i \mapsto \hat{\a}_i\) defines a (possibly approximate) projection \(\Pi_{\mathcal{C}_{\D}}: \mathbb{R}^d \to \mathcal{C}_{\D}\).

Similarly, we now show that for CBMs (Equation~\ref{eq:cbm_constraints}), the bottleneck representation \(\hat{\a}_i = \hat{\bm{c}}_i^\top \mathbf{W}_c\) lies in the concept cone \(\mathcal{C}_{\mathbf{W}_c}\).
Since in a CBM, concept predictions are \(\hat{\bm{c}}_i = \sigma(\a_i \mathbf{W}_c^\top + \bm{b}_c)\) where \(\sigma\) is the sigmoid (or other activation ensuring \([0,1]\) output).
By definition of sigmoid, \(\hat{c}_{ij} \in [0,1] \subset \mathbb{R}_+\) for all \(j\).
The bottleneck representation that is fed to the classifier is:
\begin{equation}
\hat{\a}_i = \hat{\bm{c}}_i^\top \mathbf{W}_c = \sum_{j=1}^c \hat{c}_{ij} (\mathbf{W}_c)_j
\end{equation}
where \((\mathbf{W}_c)_j\) is the \(j\)-th row of \(\mathbf{W}_c\).
Since \(\hat{c}_{ij} \geq 0\), this is a nonnegative linear combination of concept directions, hence \(\hat{\a}_i \in \mathcal{C}_{\mathbf{W}_c}\).
The mapping \(\a_i \mapsto \hat{\a}_i\) defines a projection (in the sense of dimension reduction through the bottleneck) onto the concept cone.

Both paradigms produce representations that lie in concept cones, defining implicit projection operators \(\Pi_{\mathcal{C}_{\D}}\).

\textbf{(ii): Shared geometric structure}

We start with unsupervised method and show that they select a cone \(\mathcal{C}_{\D}\) to minimize reconstruction error: \(\min_{\D} \mathbb{E}_{\a \sim P_{\A}}[\|\a - \Pi_{\mathcal{C}_{\D}}(\a)\|^2]\) subject to sparsity/structure constraints on codes.
The objective in Equation~\ref{eq:unsup_obj} can be written per-sample as:
\begin{equation}
\mathcal{L}_{\text{unsup}}(\a_i) = \min_{\bm{z}_i \geq 0} \|\a_i - \bm{z}_i^\top \D\|^2 + \Omega(\bm{z}_i)
\end{equation}
The optimal code \(\bm{z}_i^*\) produces the reconstruction \(\hat{\a}_i = \bm{z}_i^{*\top} \D \in \mathcal{C}_{\D}\).
Since \(\bm{z}_i^* = \argmin_{\bm{z} \geq 0} \|\a_i - \bm{z}^\top \D\|^2 + \Omega(\bm{z})\), the reconstruction \(\hat{\a}_i\) is the best approximation of \(\a_i\) within the cone \(\mathcal{C}_{\D}\) under the given regularization.
Optimizing over \(\D\) thus amounts to selecting the cone that minimizes expected reconstruction error over the data distribution.

We now turn to unsupervised methods and show that CBMs select a cone \(\mathcal{C}_{\mathbf{W}_c}\) to maximize task performance: \(\min_{\mathbf{W}_c} \mathbb{E}_{(\a, y) \sim P}[\mathcal{L}_{\mathrm{CE}}(h(\Pi_{\mathcal{C}_{\mathbf{W}_c}}(\a)), y)]\) where \(h\) is the task classifier.
We note that the CBM objective in Equation~\ref{eq:sup_obj} can be decomposed as:
\begin{align}
\mathcal{L}_{\text{CBM}} &= \mathbb{E}_{(\a, y)}[\mathcal{L}_{\mathrm{CE}}(\mathrm{softmax}(\hat{\bm{c}}^\top \mathbf{W}_y), y)] + \lambda \mathbb{E}_{(\a, \bm{c})}[\mathcal{L}_{\mathrm{BCE}}(\hat{\bm{c}}, \bm{c})] \\
&= \mathbb{E}_{(\a, y)}[\mathcal{L}_{\mathrm{CE}}(h(\hat{\bm{c}}^\top \mathbf{W}_c), y)] + \lambda \mathbb{E}[\mathcal{L}_{\mathrm{BCE}}(\hat{\bm{c}}, \bm{c})]
\end{align}
where \(\hat{\bm{c}}^\top \mathbf{W}_c \in \mathcal{C}_{\mathbf{W}_c}\) is the bottleneck representation.
The first term optimizes the cone \(\mathcal{C}_{\mathbf{W}_c}\) such that projections onto it preserve task-relevant information.
The second term (when \(\lambda > 0\)) biases the cone toward human-interpretable concept directions, but the primary driver is task loss.
Thus, supervised methods select the cone that best supports downstream classification, as opposed to reconstruction.

Thus, both paradigms optimize for a convex cone \(\mathcal{C}_{\D}\) in activation space, but under different criteria: unsupervised methods prioritize reconstruction fidelity, while supervised methods prioritize task discrimination.
The geometric structure—a cone generated by nonnegative combinations of concept directions—is identical, only the selection criterion differs.

\textbf{(iii): Concept recoverability}
We now show that a concept direction \(\bm{w} \in \mathbb{R}^d\) is recoverable from dictionary \(\D\) if and only if \(\bm{w} \in \mathcal{C}_{\D}\).
If \(\bm{w}\) is recoverable, then \(\bm{w} \in \mathcal{C}_{\D}\):
Suppose \(\bm{w}\) is recoverable from \(\D\), meaning there exists a code \(\bm{\alpha} \geq 0\) such that \(\bm{w} = \bm{\alpha}^\top \D\). By definition of the concept cone (Equation~\ref{eq:concept_cone}), this implies \(\bm{w} \in \mathcal{C}_{\D}\).
Similarly, if \(\bm{w} \in \mathcal{C}_{\D}\), then \(\bm{w}\) is recoverable.
We note the contrapositive: if \(\bm{w} \notin \mathcal{C}_{\D}\), then no nonnegative combination of \(\D\)'s rows can produce \(\bm{w}\), hence \(\bm{w}\) is not recoverable from \(\D\).

\end{proof}

This geometric view makes precise the intuition that both supervised and unsupervised methods "carve out" a subspace of activation space, constrained to be a convex cone due to nonnegativity, and differing only in \emph{which} cone they select based on their optimization objective.

\section{Theoretical insights on the metrics}
\label{sec:metric_theory}

The metrics introduced in the paper are designed to capture complementary aspects of the relationship between CBM and SAE concept spaces.  
While the cone formalism gives a unified geometric language, it also reveals that different notions of agreement should not be conflated.  
In particular, cone containment, sparse recoverability, atom-level alignment, activation-level agreement, and regression predictability measure distinct properties.  
We now formalize these distinctions through a series of simple propositions.

Let \(\D_{\mathrm{rand}}\) denote a random SAE dictionary.

\begin{proposition}[Recoverability monotonicity under cone inclusion] \label{prop:RecoMono}
Assume that
\[
\mathcal{C}_{\D} \subseteq \mathcal{C}_{\D_{\mathrm{rand}}}.
\]
Then, for any concept direction \(\bm v_i \in \mathbb{R}^d\), if \(\bm v_i\) is recoverable from \(\D\), then \(\bm v_i\) is recoverable from \(\D_{\mathrm{rand}}\).
\end{proposition}

\begin{proof}
If \(\bm v_i\) is recoverable from \(\D\), then by Observation~\ref{prop:concept_cone_unification}(iii),
\[
\bm v_i \in \mathcal{C}_{\D}.
\]
Since \(\mathcal{C}_{\D} \subseteq \mathcal{C}_{\D_{\mathrm{rand}}}\), it follows that
\[
\bm v_i \in \mathcal{C}_{\D_{\mathrm{rand}}}.
\]
Again by Observation~\ref{prop:concept_cone_unification}(iii), \(\bm v_i\) is recoverable from \(\D_{\mathrm{rand}}\).
\end{proof}

The first result shows that a random SAE would have a better coverage than a trained SAE. The hypotheses of this proposition are validated in Appendix \ref{subsec:ConeContainement}.  This justifies that a random SAE might have a higher coverage than a trained one.

The next proposition makes explicit why the coverage score is a natural measure of approximate containment.  
It expresses coverage as one minus the normalized residual energy of the CBM directions after projection onto the SAE cone.

\begin{proposition}[Coverage characterizes approximate cone containment]
Let \(\{\bm v_i\}_{i=1}^{c_{\mathrm{CBM}}}\) be CBM concept directions and let
\[
\bm{\alpha}_i^*
=
\argmin_{\bm{\alpha}\geq 0}
\|\bm v_i-\bm{\alpha}^{\top}\D_{\mathrm{SAE}}\|_2^2.
\]
Define
\[
\delta_i =
\frac{\|\bm v_i-\bm{\alpha}_i^{*\top}\D_{\mathrm{SAE}}\|_2}
{\|\bm v_i\|_2}.
\]
Then
\[
\mathrm{Cov}
=
1-
\frac{\sum_i \delta_i^2 \|\bm v_i\|_2^2}
{\sum_i \|\bm v_i\|_2^2},
\]
provided the projection residuals are orthogonal to the cone projection. In particular,
\(\mathrm{Cov}=1\) if and only if every CBM direction lies in
\(\mathcal{C}_{\D_{\mathrm{SAE}}}\).
\end{proposition}

\begin{proof}
Let
\[
\hat{\bm v}_i=\bm{\alpha}_i^{*\top}\D_{\mathrm{SAE}}
=
\Pi_{\mathcal{C}_{\D_{\mathrm{SAE}}}}(\bm v_i).
\]
For metric projections onto closed convex cones,
\[
\|\bm v_i\|_2^2
=
\|\hat{\bm v}_i\|_2^2
+
\|\bm v_i-\hat{\bm v}_i\|_2^2.
\]
Therefore,
\[
\|\hat{\bm v}_i\|_2^2
=
\|\bm v_i\|_2^2
-
\delta_i^2\|\bm v_i\|_2^2.
\]
Summing over \(i\) and dividing by \(\sum_i\|\bm v_i\|_2^2\) gives the result.
\end{proof}

However, containment alone is not enough to conclude that the two dictionaries discover the same axes.  
A CBM direction may be contained in the SAE cone because it is reconstructed by a combination of SAE atoms, even if no single SAE atom is aligned with that CBM direction.

\begin{proposition}[Containment does not imply axis alignment]
There exist dictionaries \(\D_1,\D_2\) such that
\[
\mathcal{C}_{\D_1}\subseteq \mathcal{C}_{\D_2},
\]
but no atom of \(\D_2\) is exactly aligned with an atom of \(\D_1\).
Consequently, high coverage does not necessarily imply high
\(\rho_{\mathrm{geom}}\).
\end{proposition}

\begin{proof}
Consider \(\mathbb{R}^2\) and let
\[
\D_1=\{(1,0)\}.
\]
For a small angle \(\theta>0\), define
\[
\D_2=\{(\cos\theta,\sin\theta),(\cos\theta,-\sin\theta)\}.
\]
Then
\[
(1,0)
=
\frac{1}{2\cos\theta}
(\cos\theta,\sin\theta)
+
\frac{1}{2\cos\theta}
(\cos\theta,-\sin\theta),
\]
so \((1,0)\in\mathcal{C}_{\D_2}\) and hence
\(\mathcal{C}_{\D_1}\subseteq\mathcal{C}_{\D_2}\).
However, neither atom of \(\D_2\) is exactly aligned with \((1,0)\) unless \(\theta=0\).
Thus containment can hold through combinations of atoms even when individual axes differ.
\end{proof}

This distinction is especially important for SAEs, since SAEs are trained to reconstruct broad activation structure, whereas CBMs focus on a smaller set of supervised, human-desirable concepts.  
Thus, we should not expect all SAE atoms to match CBM concepts.  
Instead, only a subset of SAE atoms may participate in the recovery of CBM directions, while the remaining atoms encode additional non-CBM factors.

\begin{proposition}[Partial semantic overlap]
Let the SAE dictionary decompose into two disjoint subsets
\[
\D_{\mathrm{SAE}}
=
\D_{\mathrm{CBM}} \cup \D_{\mathrm{extra}},
\]
where \(\D_{\mathrm{CBM}}\) contains atoms useful for reconstructing CBM concepts and
\(\D_{\mathrm{extra}}\) contains task-irrelevant or non-CBM factors needed for activation reconstruction.
Assume CBM concepts are recoverable from \(\D_{\mathrm{CBM}}\):
\[
\forall i,\qquad
\bm v_i\in \mathcal{C}_{\D_{\mathrm{CBM}}}.
\]
Then
\[
\mathrm{Cov}(\D_{\mathrm{SAE}})=1,
\]
even though many SAE atoms in \(\D_{\mathrm{extra}}\) may have low correlation with all CBM concepts.
\end{proposition}

\begin{proof}
Since
\[
\D_{\mathrm{CBM}}\subseteq \D_{\mathrm{SAE}},
\]
we have
\[
\mathcal{C}_{\D_{\mathrm{CBM}}}
\subseteq
\mathcal{C}_{\D_{\mathrm{SAE}}}.
\]
Because each \(\bm v_i\) lies in \(\mathcal{C}_{\D_{\mathrm{CBM}}}\), it also lies in
\(\mathcal{C}_{\D_{\mathrm{SAE}}}\).
Therefore every CBM direction is exactly recoverable from the SAE cone, so
\(\mathrm{Cov}(\D_{\mathrm{SAE}})=1\).
However, atoms in \(\D_{\mathrm{extra}}\) need not be correlated with any CBM direction, so atom-level alignment metrics averaged over all SAE atoms may be low.
\end{proof}

In practice, we often observe the contrary, this show that the $\D_{\mathrm{CBM}}$ might be bigger not reconstruct all the vectors of $\bm v_i$ of the CBM.
The previous propositions concern the geometry of the dictionaries themselves.  
However, CBMs and SAEs also induce latent representations over the dataset through their code matrices, \(\Z^{\mathrm{CBM}}\) and \(\Z^{\mathrm{SAE}}\).  
Dictionary-level agreement and activation-level agreement are not equivalent: two models may use similar directions in activation space but apply different encoders, biases, nonlinearities, thresholds, or sparsity operators. 
Therefore, high atom-level correlation,
\[
|\mathrm{corr}(\bm d_j,\bm w_i)|\approx 1,
\]
does not necessarily imply high activation correlation,
\[
|\mathrm{corr}(\Z^{\mathrm{SAE}}_{\cdot j},\Z^{\mathrm{CBM}}_{\cdot i})|\approx 1.
\]
This motivates \(\rho_{\mathrm{act}}\) and $R^2$, which measures whether geometrically related concepts also behave similarly across samples.

While $R^2$ measures the global predictability of CBM activations from SAE representations, it does not guarantee that individual SAE concepts correspond to CBM concepts. The following proposition formalizes this distinction.

\begin{proposition}[Global predictability does not imply pairwise activation alignment]
There exist nonnegative SAE and CBM activation matrices such that
\[
R^2 = 1,
\]
while
\[
\rho_{\mathrm{act}}
\]
can be arbitrarily small. Consequently, high regression predictability does not imply high pairwise activation alignment.
\end{proposition}

\begin{proof}
Let \(z_1,\ldots,z_m\) be independent Bernoulli random variables with
\[
z_\ell \in \{0,1\},
\qquad
\mathbb{E}[z_\ell]=\frac12,
\qquad
\mathrm{Var}(z_\ell)=\frac14.
\]
Thus all activations are nonnegative and have positive mean.

Define the CBM activations as
\[
\Z^{\mathrm{CBM}}
=
\begin{pmatrix}
z_1 & z_2 & \cdots & z_m
\end{pmatrix}.
\]

Now define SAE activations by nonnegative mixtures:
\[
\Z^{\mathrm{SAE}}
=
\Z^{\mathrm{CBM}} M,
\]
where
\[
M =
\mathbf{1}\mathbf{1}^{\top} + \epsilon I_m,
\]
with \(\epsilon>0\).  
Since \(M\) has strictly positive entries, \(\Z^{\mathrm{SAE}}\) is nonnegative.  
Moreover, \(M\) is invertible for every \(\epsilon>0\). Therefore,
\[
\Z^{\mathrm{CBM}}
=
\Z^{\mathrm{SAE}} M^{-1}.
\]
Hence there exists a linear map \(\Phi\) such that
\[
\Z^{\mathrm{CBM}}=\Phi(\Z^{\mathrm{SAE}}),
\]
and therefore
\[
R^2=1.
\]

We now study pairwise activation correlations.  
The \(j\)-th SAE activation is
\[
s_j
=
\sum_{\ell=1}^m z_\ell+\epsilon z_j.
\]
For any CBM concept \(z_i\), using independence,
\[
\mathrm{Cov}(s_j,z_i)
=
\begin{cases}
(1+\epsilon)\mathrm{Var}(z_i), & i=j,\\
\mathrm{Var}(z_i), & i\neq j.
\end{cases}
\]
Also,
\[
\mathrm{Var}(s_j)
=
\left((1+\epsilon)^2+(m-1)\right)\mathrm{Var}(z_i).
\]
Therefore,
\[
\max_i |\mathrm{corr}(s_j,z_i)|
=
\frac{1+\epsilon}
{\sqrt{(1+\epsilon)^2+(m-1)}}.
\]
As \(m\to\infty\), this quantity tends to \(0\). Hence
\[
\rho_{\mathrm{act}}
=
\frac{1}{m}\sum_{j=1}^m
\max_i |\mathrm{corr}(s_j,z_i)|
=
\frac{1+\epsilon}
{\sqrt{(1+\epsilon)^2+(m-1)}}
\to 0.
\]

Thus, the CBM activations are perfectly linearly predictable from the SAE activations, so \(R^2=1\), while pairwise activation alignment can be arbitrarily small.
\end{proof}

Together, these observations show why a single metric is insufficient.  
Coverage, sparsity, geometric alignment, activation alignment, and regression predictability are non-equivalent:
\begin{enumerate}[label=(\roman*)]
    \item high coverage does not imply high \(\rho_{\mathrm{geom}}\);
    \item high \(\rho_{\mathrm{geom}}\) does not imply high \(\rho_{\mathrm{act}}\);
    \item high \(\rho_{\mathrm{act}}\) does not imply sparse geometric recovery;
    \item high \(R^2\) does not imply one-to-one concept alignment.
\end{enumerate}
Thus, the proposed metrics are complementary: coverage measures containment of CBM directions in the SAE cone; sparsity measures the simplicity of the recovery; \(\rho_{\mathrm{geom}}\) measures axis-level agreement; \(\rho_{\mathrm{act}}\) measures sample-level agreement; and \(R^2\) measures global predictive transfer between the two concept spaces.

\section{Experimental settings}\label{sec:experimentalseting}
\subsection{Concept bottleneck model architecture}
The CBM branch is constructed in parallel to the SAE branch and operates on the same frozen backbone layer. Each CBM consists of two components. First, a feature-to-concept predictor \(W_{c}\) maps backbone activations to \(M\) concept logits, which are converted into probabilities through a sigmoid. Second, a bias-free linear classifier \(W_{y}\) maps the predicted concept vector to class logits. We evaluate three backbones: (i) \textbf{ResNet-50}, where the CBM/SAE is applied after global average pooling for a given layer; (ii) \textbf{ViT-B/16}; and (iii) \textbf{DINO-v2 ViT-B/14}, where both CBM and SAE operate on the class token. The models were trained on three seeds : 0, 42, and 2026.
  
\paragraph{Hardware.} All experiments were conducted on a single NVIDIA H100 GPU with 24 CPU cores allocated per job. Experiments on CUB-200 were performed in full FP32 precision, whereas ImageNet experiments used BF16 automatic mixed precision with a batch size of 2048. Each job, corresponding to a single random seed and backbone architecture, required approximately 8 to 20 hours to complete.

\paragraph{CUB.}
For CUB, we use the human-annotated bird attributes. Uncertain image-level annotations are removed and the remaining labels are aggregated into class-level binary concepts via majority vote. Following prior work \citep{koh2020concept}, we retain only attributes that appear in a sufficient number of classes. Out of 312 attributes, 168 pass this filtering. We jointly train \(W_{c}\) and \(W_{y}\) using a binary cross-entropy loss on the concepts and a cross-entropy loss on the class logits. We trained each backbone under its own optimized schedule: ResNet-50 was trained for 75 epochs with a learning rate of 0.01, ViT was trained for 50 epochs with a learning rate of 0.01, and DINO was loaded from \textbf{DINO-v2 ViT-B/14} and frozen, then a linear probing layer was trained for 50 epochs with a learning rate of 0.01. However, we plug the SAE on the CLS-token.

\paragraph{ImageNet.}
For ImageNet, we rely on the VLG-CBM concepts introduced in \citep{srivastava2024vlg}. These concepts are obtained via a VLLM annotation pipeline that also provides an uncertainty score for each concept. After thresholding by confidence and aggregating image-level labels into class-level concepts via majority vote, we retain 3545 concepts out of the original 5305 using the same strategy as that of CUB. Training again optimizes both concept prediction and class classification via the combined loss (similar to what we did on for CUB's dataset).

For each backbone, we loaded weights. For ResNet50, we use $''IMAGENET1K\_V2''$ weights. For ViT-B/16, we use $''IMAGENET1K\_V1''$ weights from Pytorch~\citep{paszke2019pytorch}. We load \textbf{DINO-v2 ViT-B/14} for the DINO backbone.

\paragraph{Concept Selection and Performance.}
The curated concept sets for both datasets allow the CBM to achieve strong predictive accuracy across backbones. Table~\ref{tab:label_free_cbm} compares our results to two baselines: P-CBM \citep{yuksekgonul2022post} and Label-Free CBM \citep{oikarinen2023label}. Our CBM consistently outperforms prior approaches, and the improvements hold across all architectures.

On \textbf{CUB}, our DINO-v2 CBM reaches \textbf{84.39\%}, outperforming the Label-Free CBM (74.31\%) by over ten points and substantially surpassing the P-CBM baseline (59.60\%). ResNet-50 and ViT-B/16 also achieve solid performance (72.79\% and 69.70\%, respectively), confirming that our concept selection and training pipeline generalizes across architectures.

On \textbf{ImageNet}, our CBM achieves \textbf{77--80\%} accuracy across backbones. DINO-v2 again yields the strongest result with \textbf{80.30\%}, clearly improving over the Label-Free CBM baseline (71.95\%). These high scores indicate that the learned concepts provide meaningful supervision even on a large and diverse dataset.

Overall, the table confirms that our CBM implementation is competitive with and often superior to previous CBM frameworks. Combined with the SAE results, this supports the conclusion that our CBMs form a reliable reference geometry for evaluating the alignment between supervised concepts and unsupervised sparse representations.

\begin{table}[htbp]
\centering
\scriptsize
\setlength{\tabcolsep}{4pt}
\renewcommand{\arraystretch}{1.05}
\caption{Comparison of CBM across datasets.}
\label{tab:label_free_cbm}
%\resizebox{0.8\columnwidth}{!}{
\begin{tabular}{l c c}
\toprule
\multicolumn{3}{c}{\textbf{Dataset}} \\
\cmidrule(lr){2-3}
\textbf{Model} & \textbf{CUB200} & \textbf{ImageNet} \\
\midrule
%Standard & 76.70\% & 76.13\% \\
%\midrule
%Standard (sparse) & 75.96\% & 74.35\% \\
P-CBM             & 59.60\% & N/A \\
%P-CBM (CLIP)      & N/A & N/A \\
\midrule
\multirow{1}{*}{\begin{tabular}[c]{@{}l@{}}Label-free CBM\\\end{tabular}}
 & 74.31\% & 71.95\% \\
 \midrule
\multirow{1}{*}{\begin{tabular}[c]{@{}l@{}}
 Our CBM on RN50\\\end{tabular}} & 72.79\% & 77.95\% \\
\multirow{1}{*}{\begin{tabular}[c]{@{}l@{}}
Our CBM on VIT\\\end{tabular}} & 69.70\% & 78.12\% \\
\multirow{1}{*}{\begin{tabular}[c]{@{}l@{}}
Our CBM on DINO\\\end{tabular}} & \textbf{84.39\%} & \textbf{80.30\%} \\
\bottomrule
\end{tabular}
\end{table}

\section{Additional results}
\subsection{ Sparsity evaluation.}\label{sec:experimentsSparsity}
The ``Sparsity'' column reports the fraction of inactive units in the learned representations, where values close to $1$ indicate highly sparse codes. Lower \texttt{target\_L0} values produce extremely sparse encodings (few active atoms), whereas higher \texttt{target\_L0} yields denser representations (more active units). We analyze how this trade-off influences alignment measures, including geometric alignment ($\rho_{\text{geom}}$), activation correlation ($\rho_{\text{act}}$), CBM-concept coverage, regression predictability ($R^2$), matching entropy, and downstream F1 performance.

Overall, increasing the number of active units (i.e., reducing sparsity) improves coverage, $R^2$, and sample-level F1, but leads to a loss in geometric concentration. This illustrates a trade-off between sparsity (parsimony of representation) and the semantic coverage required to approximate the CBM concept space. Consistent patterns appear across backbones: geometric alignment and entropy peak at higher sparsity levels (fewer active neurons), while activation correlation, coverage, $R^2$, and F1 reach their best performance at lower sparsity, where denser codes better reconstruct the CBM concept structure.

\begin{table}[htbp]
\centering
\scriptsize
\setlength{\tabcolsep}{4pt}
\renewcommand{\arraystretch}{1.05}
\caption{\texttt{TopK} SAE results (dimension = 8192) on CUB and ImageNet (IN) using ResNet-50 (R50). ``target\_L0'' is the target activation level for \texttt{TopK}; ``Sparsity'' is the observed fraction of inactive units (higher = sparser). The SAE has a $\times4$ expansion factor.}
\label{tab:topk_sae_targets}
\resizebox{0.8\columnwidth}{!}
        {
\begin{tabular}{c|c c c c c c c}
\toprule
\textbf{Dataset--Backbone} 
& \textbf{target\_L0} 
& \textbf{Sparsity} 
& $\rho_{\text{geom}}$ 
& $\rho_{\text{act}}$ 
& \textbf{Coverage} 
& $R^2$ 
& \textbf{F1} \\
\midrule

\multirow{7}{*}{\rotatebox{90}{CUB--R50}}
 & 0.0012 & $0.9988$ & $\mathbf{0.0743 \pm 0.0034}$ & $0.0342 \pm 0.0001$ & $0.1469 \pm 0.0150$ & $\mathbf{1.0000 \pm 0.0000}$ & $\mathbf{0.0987 \pm 0.0034}$ \\
 & 0.0024 & $0.9976$ & $0.0733 \pm 0.0023$ & $0.0418 \pm 0.0002$ & $0.4006 \pm 0.0204$ & $\mathbf{1.0000 \pm 0.0000}$ & $0.0878 \pm 0.0007$ \\
 & 0.0050 & $0.9950$ & $0.0721 \pm 0.0018$ & $0.0475 \pm 0.0005$ & $0.6170 \pm 0.0104$ & $\mathbf{1.0000 \pm 0.0000}$ & $0.0824 \pm 0.0031$ \\
 & 0.0120 & $0.9880$ & $0.0684 \pm 0.0013$ & $0.0524 \pm 0.0003$ & $0.7715 \pm 0.0083$ & $\mathbf{1.0000 \pm 0.0000}$ & $0.0781 \pm 0.0056$ \\
 & 0.0240 & $0.9773$ & $0.0662 \pm 0.0008$ & $0.0675 \pm 0.0024$ & $0.8014 \pm 0.0173$ & $\mathbf{1.0000 \pm 0.0000}$ & $0.0734 \pm 0.0074$ \\
 & 0.0500 & $0.9667$ & $0.0641 \pm 0.0068$ & $0.0680 \pm 0.0048$ & $0.8239 \pm 0.0320$ & $0.9989 \pm 0.0022$ & $0.0736 \pm 0.0035$ \\
 & 0.1000 & $0.9645$ & $0.0622 \pm 0.0056$ & $\mathbf{0.0700 \pm 0.0062}$ & $\mathbf{0.8373 \pm 0.0237}$ & $0.9991 \pm 0.0018$ & $0.0739 \pm 0.0048$ \\

\midrule

\multirow{7}{*}{\rotatebox{90}{IN--R50}}
 & 0.0012 & $0.9988$ & $0.1153 \pm 0.0016$ & $0.0956 \pm 0.0202$ & $0.0502 \pm 0.0026$ & $0.7437 \pm 0.0099$ & $0.1431 \pm 0.0004$ \\
 & 0.0024 & $0.9976$ & $0.1149 \pm 0.0018$ & $0.1192 \pm 0.0062$ & $0.2297 \pm 0.0185$ & $0.7461 \pm 0.0041$ & $0.1370 \pm 0.0003$ \\
 & 0.0050 & $0.9950$ & $0.1220 \pm 0.0023$ & $0.1496 \pm 0.0005$ & $0.7203 \pm 0.0151$ & $0.7761 \pm 0.0015$ & $0.1359 \pm 0.0001$ \\
 & 0.0120 & $0.9880$ & $\mathbf{0.1255 \pm 0.0002}$ & $0.1667 \pm 0.0026$ & $0.9368 \pm 0.0027$ & $\mathbf{0.7942 \pm 0.0042}$ & $0.1380 \pm 0.0002$ \\
 & 0.0240 & $0.9760$ & $0.1223 \pm 0.0012$ & $\mathbf{0.1685 \pm 0.0007}$ & $0.9574 \pm 0.0028$ & $0.7940 \pm 0.0024$ & $0.1360 \pm 0.0030$ \\
 & 0.0500 & $0.9500$ & $0.1007 \pm 0.0006$ & $0.1605 \pm 0.0005$ & $\mathbf{0.9650 \pm 0.0001}$ & $0.7787 \pm 0.0021$ & $0.1409 \pm 0.0005$ \\
 & 0.1000 & $0.9000$ & $0.0926 \pm 0.0007$ & $0.1514 \pm 0.0005$ & $0.9626 \pm 0.0005$ & $0.7617 \pm 0.0007$ & $\mathbf{0.1469 \pm 0.0005}$ \\

\midrule

\multirow{7}{*}{\rotatebox{90}{CUB--ViT16B}}
 & 0.0012 & $0.9987$ & $\mathbf{0.0835 \pm 0.0047}$ & $0.0256 \pm 0.0002$ & $0.1309 \pm 0.0081$ & $0.7529 \pm 0.0111$ & $\mathbf{0.1274 \pm 0.0029}$ \\
 & 0.0024 & $0.9977$ & $0.0777 \pm 0.0049$ & $0.0386 \pm 0.0009$ & $0.2426 \pm 0.0064$ & $0.8104 \pm 0.0082$ & $0.1273 \pm 0.0025$ \\
 & 0.0050 & $0.9951$ & $0.0761 \pm 0.0048$ & $0.0550 \pm 0.0016$ & $0.4161 \pm 0.0157$ & $0.8667 \pm 0.0055$ & $0.1225 \pm 0.0013$ \\
 & 0.0120 & $0.9880$ & $0.0792 \pm 0.0042$ & $0.0699 \pm 0.0014$ & $0.6618 \pm 0.0359$ & $0.9112 \pm 0.0060$ & $0.1072 \pm 0.0010$ \\
 & 0.0240 & $0.9759$ & $0.0821 \pm 0.0032$ & $0.0733 \pm 0.0014$ & $0.7545 \pm 0.0136$ & $0.9281 \pm 0.0078$ & $0.0868 \pm 0.0018$ \\
 & 0.0500 & $0.9499$ & $0.0823 \pm 0.0015$ & $0.0790 \pm 0.0039$ & $0.7752 \pm 0.0086$ & $0.9437 \pm 0.0079$ & $0.0782 \pm 0.0030$ \\
 & 0.1000 & $0.9182$ & $0.0796 \pm 0.0051$ & $\mathbf{0.1079 \pm 0.0174}$ & $\mathbf{0.7957 \pm 0.1044}$ & $\mathbf{0.9598 \pm 0.0073}$ & $0.0854 \pm 0.0100$ \\

\midrule

\multirow{7}{*}{\rotatebox{90}{IN--ViT16B}}
 & 0.0012 & $0.9987$ & $\mathbf{0.0925 \pm 0.0002}$ & $0.1745 \pm 0.0006$ & $0.0074 \pm 0.0001$ & $0.7815 \pm 0.0012$ & $0.1717 \pm 0.0001$ \\
 & 0.0024 & $0.9977$ & $0.0762 \pm 0.0004$ & $0.2109 \pm 0.0017$ & $0.0114 \pm 0.0011$ & $0.7863 \pm 0.0011$ & $0.1705 \pm 0.0001$ \\
 & 0.0050 & $0.9951$ & $0.0612 \pm 0.0002$ & $0.2748 \pm 0.0018$ & $0.0401 \pm 0.0044$ & $0.7976 \pm 0.0006$ & $0.1694 \pm 0.0001$ \\
 & 0.0120 & $0.9880$ & $0.0546 \pm 0.0002$ & $\mathbf{0.3174 \pm 0.0015}$ & $0.6585 \pm 0.0091$ & $\mathbf{0.8048 \pm 0.0001}$ & $0.1713 \pm 0.0001$ \\
 & 0.0240 & $0.9759$ & $0.0556 \pm 0.0004$ & $0.3128 \pm 0.0009$ & $0.8437 \pm 0.0011$ & $0.8014 \pm 0.0005$ & $\mathbf{0.1723 \pm 0.0008}$ \\
 & 0.0500 & $0.9499$ & $0.0616 \pm 0.0002$ & $0.3069 \pm 0.0002$ & $0.9099 \pm 0.0038$ & $0.8027 \pm 0.0005$ & $0.1703 \pm 0.0010$ \\
 & 0.1000 & $0.9001$ & $0.0712 \pm 0.0072$ & $0.2916 \pm 0.0334$ & $\mathbf{0.9564 \pm 0.0067}$ & $0.7950 \pm 0.0295$ & $0.1590 \pm 0.0115$ \\

\midrule

\multirow{7}{*}{\rotatebox{90}{CUB--DINO}}
 & 0.0012 & $0.9987$ & $\mathbf{0.1334 \pm 0.0113}$ & $0.0359 \pm 0.0004$ & $0.0810 \pm 0.0210$ & $0.8549 \pm 0.0220$ & $0.1206 \pm 0.0009$ \\
 & 0.0024 & $0.9977$ & $0.1150 \pm 0.0096$ & $0.0472 \pm 0.0005$ & $0.1590 \pm 0.0264$ & $0.8913 \pm 0.0166$ & $\mathbf{0.1221 \pm 0.0009}$ \\
 & 0.0050 & $0.9951$ & $0.1026 \pm 0.0072$ & $0.0610 \pm 0.0005$ & $0.3672 \pm 0.0201$ & $0.9179 \pm 0.0151$ & $0.1209 \pm 0.0020$ \\
 & 0.0120 & $0.9880$ & $0.0964 \pm 0.0047$ & $0.0690 \pm 0.0007$ & $0.7424 \pm 0.0120$ & $0.9344 \pm 0.0131$ & $0.1131 \pm 0.0024$ \\
 & 0.0240 & $0.9759$ & $0.0932 \pm 0.0022$ & $0.0713 \pm 0.0005$ & $0.8353 \pm 0.0310$ & $0.9452 \pm 0.0123$ & $0.1106 \pm 0.0025$ \\
 & 0.0500 & $0.9499$ & $0.0883 \pm 0.0002$ & $0.0743 \pm 0.0009$ & $0.8337 \pm 0.0389$ & $0.9567 \pm 0.0109$ & $0.1089 \pm 0.0018$ \\
 & 0.1000 & $0.9181$ & $0.0887 \pm 0.0027$ & $\mathbf{0.0893 \pm 0.0014}$ & $\mathbf{0.8448 \pm 0.0273}$ & $\mathbf{0.9610 \pm 0.0086}$ & $0.1061 \pm 0.0038$ \\

\midrule

\multirow{7}{*}{\rotatebox{90}{IN--DINO}}
 & 0.0012 & $0.9987$ & $\mathbf{0.1773 \pm 0.0006}$ & $0.1871 \pm 0.0009$ & $0.0075 \pm 0.0000$ & $0.7749 \pm 0.0016$ & $\mathbf{0.1536 \pm 0.0003}$ \\
 & 0.0024 & $0.9977$ & $0.1425 \pm 0.0007$ & $0.2294 \pm 0.0006$ & $0.0110 \pm 0.0005$ & $0.7909 \pm 0.0005$ & $0.1524 \pm 0.0006$ \\
 & 0.0050 & $0.9951$ & $0.1047 \pm 0.0007$ & $0.2829 \pm 0.0011$ & $0.0452 \pm 0.0026$ & $0.8023 \pm 0.0003$ & $0.1494 \pm 0.0001$ \\
 & 0.0120 & $0.9880$ & $0.0780 \pm 0.0001$ & $0.3227 \pm 0.0008$ & $0.5258 \pm 0.0311$ & $\mathbf{0.8064 \pm 0.0004}$ & $0.1488 \pm 0.0001$ \\
 & 0.0240 & $0.9759$ & $0.0741 \pm 0.0001$ & $0.3262 \pm 0.0007$ & $0.8382 \pm 0.0041$ & $0.8022 \pm 0.0002$ & $0.1491 \pm 0.0001$ \\
 & 0.0500 & $0.9499$ & $0.0727 \pm 0.0001$ & $\mathbf{0.3306 \pm 0.0002}$ & $0.8932 \pm 0.0025$ & $0.8001 \pm 0.0003$ & $0.1491 \pm 0.0003$ \\
 & 0.1000 & $0.9001$ & $0.0737 \pm 0.0001$ & $0.3183 \pm 0.0009$ & $\mathbf{0.9230 \pm 0.0011}$ & $0.8009 \pm 0.0003$ & $0.1476 \pm 0.0003$ \\

\bottomrule
\end{tabular}
}
\end{table}

\subsection{Layer-wise evaluation of \texttt{TopK} SAEs.}\label{sec:experimentsLayer_wise}

For the ResNet-50 backbone, we conduct a layer-wise analysis by extracting features at each convolutional block and attaching a \texttt{TopK} SAE to each representation. The evaluated layer is reported in the ``Layer'' column of the table. Metrics are computed independently at every layer.

On the CUB-200 dataset, both $\rho_{\text{geom}}$ and $\rho_{\text{act}}$ achieve their highest values in shallower layers, whereas the remaining metrics (coverage, $R^2$, entropy, and F1) achieve their best performance in deeper layers. A similar trend is observed for ImageNet, except that $\rho_{\text{act}}$ peaks in deeper layers for this dataset. The stronger geometric alignment in early layers is likely due to their lower representational specificity: features are more generic, enabling the SAE to recover fewer but clearer latent directions. Activation correlation follows this trend on CUB, but on ImageNet the CBM concepts are more fine-grained, making deeper layers better aligned with them.

\begin{table}[htbp]
\centering
\scriptsize
\setlength{\tabcolsep}{3pt}
\renewcommand{\arraystretch}{1.08}
\caption{Layer-wise evaluation of \texttt{TopK} SAEs trained on CUB and ImageNet (IN) with ResNet-50 (R50) features with a $\times4$ expansion factor. Values are reported as mean $\pm$ standard deviation over three seeds. Sparsity is computed as $1-\mathrm{L0\ fraction}$. Metrics assess geometric and activation alignment ($\rho_{\text{geom}}$, $\rho_{\text{act}}$), coverage of CBM concepts, predictability ($R^2$), and downstream TopK F1 performance.}
\label{tab:layerwise_topk_sae}
\resizebox{\linewidth}{!}{
\begin{tabular}{c|c c c c c c c}
\toprule
\textbf{Dataset--Backbone} 
& \textbf{Layer (Dim)} 
& \textbf{Sparsity} 
& $\rho_{\text{geom}}$ 
& $\rho_{\text{act}}$ 
& \textbf{Coverage} 
& $R^2$ 
& \textbf{F1} \\
\midrule

\multirow{4}{*}{\rotatebox{90}{CUB--R50}}
 & Layer 1 (256) 
 & $0.995$
 & $\mathbf{0.1989 \pm 0.0393}$ 
 & $\mathbf{0.0504 \pm 0.0031}$ 
 & $0.0217 \pm 0.0046$ 
 & $0.7173 \pm 0.0208$ 
 & $0.0161 \pm 0.0071$ \\

 & Layer 2 (512) 
 & $0.995$ 
 & $0.0926 \pm 0.0076$ 
 & $0.0482 \pm 0.0019$ 
 & $0.0106 \pm 0.0009$ 
 & $0.7721 \pm 0.0066$ 
 & $0.0326 \pm 0.0177$ \\

 & Layer 3 (1024) 
 & $0.995$ 
 & $0.0802 \pm 0.0041$ 
 & $0.0414 \pm 0.0004$ 
 & $0.0517 \pm 0.0024$ 
 & $0.8777 \pm 0.0028$ 
 & $0.0233 \pm 0.0029$ \\

 & Layer 4 (2048) 
 & $0.995$ 
 & $0.0735 \pm 0.0036$ 
 & $0.0473 \pm 0.0003$ 
 & $\mathbf{0.6084 \pm 0.0077}$ 
 & $\mathbf{1.0000 \pm 0.0000}$ 
 & $\mathbf{0.0775 \pm 0.0046}$ \\

\midrule

\multirow{4}{*}{\rotatebox{90}{IN--R50}}
 & Layer 1 (256) 
 & $0.995$ 
 & $\mathbf{0.2472 \pm 0.0236}$ 
 & $0.0350 \pm 0.0168$ 
 & $0.0412 \pm 0.0144$ 
 & $0.1818 \pm 0.0201$ 
 & $0.0001 \pm 0.0001$ \\

 & Layer 2 (512) 
 & $0.995$ 
 & $0.2169 \pm 0.0163$ 
 & $0.0383 \pm 0.0071$ 
 & $0.0243 \pm 0.0057$ 
 & $0.1757 \pm 0.0159$ 
 & $0.0325 \pm 0.0157$ \\

 & Layer 3 (1024) 
 & $0.995$ 
 & $0.1682 \pm 0.0049$ 
 & $0.0455 \pm 0.0028$ 
 & $0.0230 \pm 0.0034$ 
 & $0.3077 \pm 0.0047$ 
 & $0.0348 \pm 0.0040$ \\

 & Layer 4 (2048) 
 & $0.995$ 
 & $0.1239 \pm 0.0015$ 
 & $\mathbf{0.1506 \pm 0.0008}$ 
 & $\mathbf{0.7131 \pm 0.0024}$ 
 & $\mathbf{0.7789 \pm 0.0024}$ 
 & $\mathbf{0.1359 \pm 0.0002}$ \\

\bottomrule
\end{tabular}
}
\end{table}
\subsection{Expansion factor analysis of \texttt{TopK} SAEs.}
\label{sec:experimentsExpansionfactor}
We also evaluate the effect of the expansion factor across all backbones for both CUB and ImageNet. The tested expansion factors, reported in the ``Expansion'' column, range from $\times1$ to $\times8$, with the corresponding SAE latent dimensions shown in brackets.

Across architectures and datasets, we observe consistent behavior: geometric alignment ($\rho_{\text{geom}}$) and activation correlation ($\rho_{\text{act}}$) are strongest at lower expansion factors, corresponding to fewer distinct concept directions. Conversely, larger latent spaces, achieved through higher expansion factors, tend to improve coverage, $R^2$, entropy, and F1, indicating that increased capacity helps capture a broader and more detailed concept space.

\begin{table}[htbp]
\centering
\scriptsize
\setlength{\tabcolsep}{4pt}
\renewcommand{\arraystretch}{1.05}
\caption{Effect of the expansion factor on \texttt{TopK} SAEs trained with ResNet-50 (R50), ViT-B/16, and DINOv2 features for CUB and ImageNet (IN). Increasing the expansion factor enlarges the latent space dimensionality. Values are reported as mean $\pm$ standard deviation over three seeds. Metrics evaluate geometric and activation alignment ($\rho_{\text{geom}}$, $\rho_{\text{act}}$), concept coverage, predictability ($R^2$), and downstream TopK F1.}
\label{tab:expansion_topk_sae}
\resizebox{\linewidth}{!}{
\begin{tabular}{c|c c c c c c c}
\toprule
\textbf{Dataset--Backbone} 
& \textbf{Expansion} 
& \textbf{Sparsity} 
& $\rho_{\text{geom}}$ 
& $\rho_{\text{act}}$ 
& \textbf{Coverage} 
& $R^2$ 
& \textbf{F1} \\
\midrule

\multirow{4}{*}{\rotatebox{90}{CUB--R50}}
 & $\times$1 (2048)  
 & $0.9951$ 
 & $\mathbf{0.0873 \pm 0.0079}$ 
 & $\mathbf{0.0546 \pm 0.0006}$ 
 & $0.0699 \pm 0.0138$ 
 & $0.8792 \pm 0.0121$ 
 & $\mathbf{0.1067 \pm 0.0025}$ \\

 & $\times$2 (4096)  
 & $0.9951$ 
 & $0.0751 \pm 0.0067$ 
 & $0.0517 \pm 0.0005$ 
 & $0.2341 \pm 0.0323$ 
 & $0.9517 \pm 0.0078$ 
 & $0.0907 \pm 0.0042$ \\

 & $\times$4 (8192)  
 & $0.9950$ 
 & $0.0736 \pm 0.0046$ 
 & $0.0475 \pm 0.0007$ 
 & $0.6077 \pm 0.0098$ 
 & $\mathbf{1.0000 \pm 0.0000}$ 
 & $0.0800 \pm 0.0048$ \\

 & $\times$8 (16384) 
 & $0.9950$ 
 & $0.0713 \pm 0.0008$ 
 & $0.0413 \pm 0.0007$ 
 & $\mathbf{0.7877 \pm 0.0241}$ 
 & $\mathbf{1.0000 \pm 0.0000}$ 
 & $0.0790 \pm 0.0084$ \\

\midrule

\multirow{4}{*}{\rotatebox{90}{IN--R50}}
 & $\times$1 (2048)  
 & $0.9951$ 
 & $\mathbf{0.1259 \pm 0.0002}$ 
 & $\mathbf{0.2072 \pm 0.0102}$ 
 & $0.0339 \pm 0.0011$ 
 & $0.6545 \pm 0.0040$ 
 & $0.1352 \pm 0.0004$ \\

 & $\times$2 (4096)  
 & $0.9951$ 
 & $0.1183 \pm 0.0003$ 
 & $0.1651 \pm 0.0108$ 
 & $0.1666 \pm 0.0095$ 
 & $0.7111 \pm 0.0041$ 
 & $0.1371 \pm 0.0001$ \\

 & $\times$4 (8192)  
 & $0.9950$ 
 & $0.1239 \pm 0.0015$ 
 & $0.1506 \pm 0.0008$ 
 & $0.7131 \pm 0.0024$ 
 & $0.7789 \pm 0.0024$ 
 & $0.1359 \pm 0.0002$ \\

 & $\times$8 (16384) 
 & $0.9950$ 
 & $0.1171 \pm 0.0003$ 
 & $0.1170 \pm 0.0022$ 
 & $\mathbf{0.9372 \pm 0.0006}$ 
 & $\mathbf{0.8436 \pm 0.0023}$ 
 & $\mathbf{0.1373 \pm 0.0001}$ \\

\midrule

\multirow{4}{*}{\rotatebox{90}{CUB--ViT}}
 & $\times$1 (768) 
 & $0.9948$ 
 & $\mathbf{0.0917 \pm 0.0084}$ 
 & $\mathbf{0.0685 \pm 0.0007}$ 
 & $0.0788 \pm 0.0131$ 
 & $0.7312 \pm 0.0158$ 
 & $0.1258 \pm 0.0043$ \\

 & $\times$2 (1536) 
 & $0.9948$ 
 & $0.0864 \pm 0.0101$ 
 & $0.0626 \pm 0.0012$ 
 & $0.2005 \pm 0.0211$ 
 & $0.8076 \pm 0.0246$ 
 & $\mathbf{0.1299 \pm 0.0052}$ \\

 & $\times$4 (3072) 
 & $0.9951$ 
 & $0.0806 \pm 0.0058$ 
 & $0.0546 \pm 0.0016$ 
 & $0.4381 \pm 0.0230$ 
 & $0.8745 \pm 0.0177$ 
 & $0.1250 \pm 0.0004$ \\

 & $\times$8 (6144) 
 & $0.9950$ 
 & $0.0866 \pm 0.0002$ 
 & $0.0459 \pm 0.0012$ 
 & $\mathbf{0.7996 \pm 0.0129}$ 
 & $\mathbf{0.9774 \pm 0.0052}$ 
 & $0.1133 \pm 0.0011$ \\

\midrule

\multirow{4}{*}{\rotatebox{90}{IN--ViT}}
 & $\times$1 (768) 
 & $0.9948$ 
 & $0.0647 \pm 0.0001$ 
 & $\mathbf{0.4586 \pm 0.0009}$ 
 & $0.0050 \pm 0.0000$ 
 & $0.6717 \pm 0.0011$ 
 & $0.1586 \pm 0.0007$ \\

 & $\times$2 (1536) 
 & $0.9948$ 
 & $\mathbf{0.0686 \pm 0.0002}$ 
 & $0.3351 \pm 0.0007$ 
 & $0.0127 \pm 0.0003$ 
 & $0.7630 \pm 0.0009$ 
 & $\mathbf{0.1719 \pm 0.0003}$ \\

 & $\times$4 (3072) 
 & $0.9951$ 
 & $0.0612 \pm 0.0001$ 
 & $0.2733 \pm 0.0012$ 
 & $0.0411 \pm 0.0033$ 
 & $0.7968 \pm 0.0010$ 
 & $0.1692 \pm 0.0004$ \\

 & $\times$8 (6144) 
 & $0.9950$ 
 & $0.0549 \pm 0.0000$ 
 & $0.2387 \pm 0.0001$ 
 & $\mathbf{0.7631 \pm 0.0044}$ 
 & $\mathbf{0.8335 \pm 0.0003}$ 
 & $0.1694 \pm 0.0005$ \\

\midrule

\multirow{4}{*}{\rotatebox{90}{CUB--DINO}}
 & $\times$1 (768) 
 & $0.9948$ 
 & $\mathbf{0.2067 \pm 0.0183}$ 
 & $\mathbf{0.0716 \pm 0.0006}$ 
 & $0.0747 \pm 0.0209$ 
 & $0.7937 \pm 0.0267$ 
 & $\mathbf{0.1261 \pm 0.0017}$ \\

 & $\times$2 (1536) 
 & $0.9948$ 
 & $0.1348 \pm 0.0103$ 
 & $0.0668 \pm 0.0007$ 
 & $0.1426 \pm 0.0249$ 
 & $0.8592 \pm 0.0207$ 
 & $0.1240 \pm 0.0023$ \\

 & $\times$4 (3072) 
 & $0.9951$ 
 & $0.1057 \pm 0.0059$ 
 & $0.0608 \pm 0.0008$ 
 & $0.3814 \pm 0.0143$ 
 & $0.9255 \pm 0.0124$ 
 & $0.1213 \pm 0.0027$ \\

 & $\times$8 (6144) 
 & $0.9950$ 
 & $0.0982 \pm 0.0036$ 
 & $0.0514 \pm 0.0001$ 
 & $\mathbf{0.7974 \pm 0.0186}$ 
 & $\mathbf{1.0000 \pm 0.0000}$ 
 & $0.1168 \pm 0.0036$ \\

\midrule

\multirow{4}{*}{\rotatebox{90}{IN--DINO}}
 & $\times$1 (768) 
 & $0.9948$ 
 & $\mathbf{0.1177 \pm 0.0003}$ 
 & $\mathbf{0.4610 \pm 0.0010}$ 
 & $0.0049 \pm 0.0001$ 
 & $0.6943 \pm 0.0019$ 
 & $0.1460 \pm 0.0009$ \\

 & $\times$2 (1536) 
 & $0.9948$ 
 & $0.1126 \pm 0.0014$ 
 & $0.3566 \pm 0.0018$ 
 & $0.0099 \pm 0.0001$ 
 & $0.7666 \pm 0.0007$ 
 & $\mathbf{0.1513 \pm 0.0005}$ \\

 & $\times$4 (3072) 
 & $0.9951$ 
 & $0.1041 \pm 0.0007$ 
 & $0.2845 \pm 0.0028$ 
 & $0.0482 \pm 0.0023$ 
 & $0.8031 \pm 0.0007$ 
 & $0.1494 \pm 0.0002$ \\

 & $\times$8 (6144) 
 & $0.9950$ 
 & $0.0941 \pm 0.0004$ 
 & $0.2609 \pm 0.0016$ 
 & $\mathbf{0.7021 \pm 0.0170}$ 
 & $\mathbf{0.8413 \pm 0.0004}$ 
 & $0.1470 \pm 0.0003$ \\

\bottomrule
\end{tabular}
}
\end{table}
\section{Why random SAEs exhibit higher coverage}\label{sec:SanityV2}

\subsection{Convex hull dispersion.}
We observe that random SAEs consistently achieve higher Coverage scores than  trained SAEs. This effect does not indicate better semantic alignment, but rather reflects a geometric artifact arising from the lack of learned structure in random representations.

A trained SAE learns a structured and compact representation space: its concepts are organized around meaningful semantic directions, resulting in a relatively narrow distribution of concept prototypes. As a consequence, the convex hull spanned by these learned concepts is constrained and occupies a smaller volume in representation space. While this structure improves interpretability and alignment with downstream CBMs, it naturally limits coverage.

In contrast, a random SAE does not learn any semantic structure. Its concept vectors are effectively random directions in a high-dimensional space, leading to a much more dispersed configuration. As a result, the convex hull formed by these random prototypes is significantly larger, which might increase the Coverage. This leads to high Coverage scores, despite the absence of meaningful semantic alignment.

Directly computing the convex hull in the SAE concept space is intractable due to its high dimensionality (4096 dimensions). To approximate the relative size of the convex hull, we instead analyze the pairwise geometric dispersion of concept prototypes. Specifically, we compute the Gram matrix of the concept vectors and measure the pairwise $\ell_2$ distances between all prototype pairs. The median of these distances serves as a proxy for the overall spread of the concept space: larger median distances indicate a more expansive convex hull.

Table~\ref{tab:convex_hull} reports the median pairwise distances for random and trained SAEs. We observe that the random SAE exhibits the largest median distance, confirming that its concept space is the most dispersed. Trained SAEs (on ImageNet and CUB) show more compact representations, consistent with their learned semantic structure.

\begin{table}[t]
\centering
\caption{Median pairwise $\ell_2$ distance between SAE concept prototypes, used as a proxy for convex hull size.}
\label{tab:convex_hull}
\begin{tabular}{l c}
\toprule
\textbf{SAE Type} & \textbf{Median Distance} \\
\midrule
Random SAE & 1.4142 \\
CUB-trained SAE & 1.3900 \\
ImageNet-trained SAE & 1.4051 \\
\bottomrule
\end{tabular}
\end{table}

These results support our interpretation that random SAEs span a larger convex hull due to unstructured and isotropic representations.

\subsection{Asymmetry of cone containment.}\label{subsec:ConeContainement}

We provide an experimental setup to show the asymmetric inclusion of Proposition \ref{prop:RecoMono}. For each backbone and each dataset, we train a \texttt{TopK} SAE and instantiate a randomly initialized SAE of the same architecture. We then compute the Coverage (eq. \ref{coverage}) between the trained SAE (with dictionary $D_{trained}$) and the random SAE (with dictionary $D_{rand}$), and symmetric calculation; inverting the role of $D_{trained}$ and $D_{rand}$ in the cover. 

Table \ref{tab:cone_asymmetry} shows a consistent asymmetry across all backbones and datasets. This table shows that the trained cone can be well approximated by the random cone. On the other hand, the random cone is poorly approximated by the trained one. This asymmetry validates Proposition \ref{prop:RecoMono}: the random dictionary has a strictly larger cone, while the trained dictionary spans a smaller, more constructed cone.

\begin{table}[h]
\centering
\caption{Empirical asymmetry of cone containment between trained and random \texttt{TopK} SAEs. Values are reported as mean $\pm$ standard deviation over three seeds.}
\label{tab:cone_asymmetry}
\begin{tabular}{llcc}
\toprule
Backbone & Dataset & Cov(trained $\to$ random) & Cov(random $\to$ trained) \\
\midrule
DINOv2 ViT-B/14 & CUB      & $0.626 \pm 0.003$ & $0.113 \pm 0.002$ \\
ResNet-50       & CUB      & $0.515 \pm 0.005$ & $0.096 \pm 0.005$ \\
ViT-B/16        & CUB      & $0.625 \pm 0.005$ & $0.262 \pm 0.008$ \\
\midrule
DINOv2 ViT-B/14 & ImageNet & $0.623 \pm 0.004$ & $0.267 \pm 0.000$ \\
ResNet-50       & ImageNet & $0.508 \pm 0.001$ & $0.139 \pm 0.002$ \\
ViT-B/16        & ImageNet & $0.624 \pm 0.003$ & $0.450 \pm 0.001$ \\
\bottomrule
\end{tabular}
\end{table}

\section{Comparison of different CBMs}\label{Appendix:dif:CBM}

\subsection{Independent, joint and sequential.}\label{Appendix:indjointseq}
The evaluation framework is robust to the choice of CBM used as a reference. We train the three CBM variants introduced by Pang Wei Koh et al.~\cite{koh2020concept}—Independent, Joint, and Sequential CBMs—using a ResNet-50 backbone on the CUB dataset. Results are reported in Table~\ref{tab:sae_comparison_cbm_types}.

Overall, the relative ranking of SAE variants remains stable across CBM types, indicating that our evaluation framework is not strongly dependent on the specific supervised reference used. Metrics obtained with Independent and Sequential CBMs are numerically very similar across all SAE variants. In contrast, Joint CBMs lead to higher alignment scores with SAE representations and higher $R^2$, while producing lower coverage values.

\begin{table}[t]
\centering
\footnotesize
\caption{Comparison of different SAE variants on CUB using ResNet-50 for three CBM training strategies.}
\label{tab:sae_comparison_cbm_types}
\resizebox{\linewidth}{!}{
\begin{tabular}{c|lcccc}
\toprule
\textbf{CBM Type}
& \textbf{SAE Type}
& \textbf{Coverage}
& $R^2$
& $\mathrm{Top}\text{-}k~F_1$
& $\rho_{\mathrm{uni}}$ \\
\midrule

\multirow{6}{*}{\rotatebox{90}{Independent}}
& \texttt{Vanilla}
& $\mathbf{0.8323}$
& $0.9240$
& $0.0652$
& $\mathbf{0.0758}$ \\

& \texttt{TopK}
& $0.1177$
& $0.9022$
& $0.0796$
& $0.0513$ \\

& \texttt{BatchTopK}
& $0.6321$
& $0.9088$
& $0.0779$
& $0.0537$ \\

& \texttt{Jump}
& $0.8291$
& $\mathbf{0.9250}$
& $0.0625$
& $0.0732$ \\

& \texttt{MP}
& $0.8202$
& $0.5058$
& $0.0496$
& $0.0274$ \\

& \texttt{Archetypal}
& $0.5484$
& $0.9051$
& $\mathbf{0.0862}$
& $0.0553$ \\

\midrule

\multirow{6}{*}{\rotatebox{90}{Joint}}
& \texttt{Vanilla}
& $\mathbf{0.7031}$
& $0.9623$
& $0.0672$
& $\mathbf{0.0829}$ \\

& \texttt{TopK}
& $0.1780$
& $0.9437$
& $0.0861$
& $0.0610$ \\

& \texttt{BatchTopK}
& $0.5694$
& $0.9491$
& $0.0857$
& $0.0643$ \\

& \texttt{Jump}
& $0.6996$
& $\mathbf{0.9627}$
& $0.0667$
& $0.0801$ \\

& \texttt{MP}
& $0.6747$
& $0.6492$
& $0.0538$
& $0.0279$ \\

& \texttt{Archetypal}
& $0.4732$
& $0.9435$
& $\mathbf{0.0932}$
& $0.0673$ \\

\midrule

\multirow{6}{*}{\rotatebox{90}{Sequential}}
& \texttt{Vanilla}
& $\mathbf{0.8324}$
& $0.9248$
& $0.0653$
& $\mathbf{0.0760}$ \\

& \texttt{TopK}
& $0.1176$
& $0.9033$
& $0.0815$
& $0.0514$ \\

& \texttt{BatchTopK}
& $0.6318$
& $0.9098$
& $0.0790$
& $0.0538$ \\

& \texttt{Jump}
& $0.8292$
& $\mathbf{0.9257}$
& $0.0638$
& $0.0734$ \\

& \texttt{MP}
& $0.8204$
& $0.5111$
& $0.0519$
& $0.0274$ \\

& \texttt{Archetypal}
& $0.5482$
& $0.9060$
& $\mathbf{0.0872}$
& $0.0554$ \\

\bottomrule
\end{tabular}
}
\end{table}

\subsection{Random CBM}\label{appendix:baselineCBM}

We compute the metrics on a trained \texttt{TopK} SAE using a ResNet-50 backbone with the CUB dataset against a randomly initialized CBM containing the same number of concepts as the standard CBM trained on CUB. Results are reported in Table \ref{tab:sanity_random_cbm}.  The goal of this experiment is to validate the utility of human-labeled concepts.

In Table \ref{tab:sanity_random_cbm},  normal refers to a trained SAE with a trained CBM, while random SAE refers to a trained CBM against a random SAE.

The random CBM has a much lower coverage, F1, and $\rho_{uni}$. The $R^2$ drops considerably. These results show that the random CBM obtains the weakest alignment with the SAE, confirming the role played by human-provided concepts in the CBM references used across experiments. The metrics are sensitive to semantic structure and depend on human-labeled concepts learned by the CBM to reach a high score. 

This sanity check supports the use of human-labeled concepts and CBMs as a semantic anchor.

\begin{table}[t]
\centering
\scriptsize
\setlength{\tabcolsep}{4pt}
\renewcommand{\arraystretch}{1.05}
\caption{Sanity-check metrics for ResNet-50 on CUB with random CBM and random SAE variants.}
\label{tab:sanity_random_cbm}
\begin{tabular}{l l l c c c c c}
\toprule
\textbf{Method} &
\textbf{Dataset} &
\textbf{Backbone} &
\textbf{Sparsity} &
$\mathrm{Coverage}$ &
$R^2$ &
$\mathrm{Top}\text{-}k~F_1$ &
$\rho_{\mathrm{uni}}$ \\
\midrule

random CBM &
CUB & R50 &
$0.005$ &
$0.0983$ &
$0.7066$ &
$0.0105$ &
$0.0401$ \\

normal &
CUB & R50 &
$0.005$ &
$0.2483$ &
$\mathbf{0.9566}$ &
$\mathbf{0.0920}$ &
$\mathbf{0.0829}$ \\

random SAE &
CUB & R50 &
$0.005$ &
$\mathbf{0.6500}$ &
$0.8503$ &
$0.0599$ &
$0.0600$ \\

\bottomrule
\end{tabular}
\end{table}

\section{Evolution of metrics across training}\label{Apendix:training}

To visualize the evolution of the metrics across the training, we train a CBM on the CUB datasets with a ResNet-50 backbone. We then compute the metrics with this CBM every 5 epochs of the training of the \texttt{TopK} SAE. 

Figure \ref{fig:Metrics_evolution} reports the evolution of the metrics during the training. The training loss decreases gradually across the training, showing that the SAE progressively improves its reconstruction. While $R^2$, topk F1, and $\rho_{uni}$ improve overall, indicating that the SAE activations become increasingly predictive of the CBM concepts.

In contrast, the behavior of the coverage confirms that an untrained or weakly structured dictionary can obtain a high apparent coverage without yielding meaningful concept correspondence. This observation is consistent with the sanity check and the results of Appendix \ref{sec:SanityV2}. 

The alignment ($\rho_{uni}$) rises over the training, suggesting that the training produces atoms that are directionally closer to CBM concepts, and the optimization redistributes information across more atoms that activate on examples similar to those activating the CBM concepts.

This Figure confirms that reconstruction quality and semantic alignment do not evolve identically and that several complementary metrics are needed to assess the SAE to CBM correspondence. 

\begin{figure*}[t]
\vspace{-5mm}
\begin{center}
\includegraphics[width=1\linewidth]{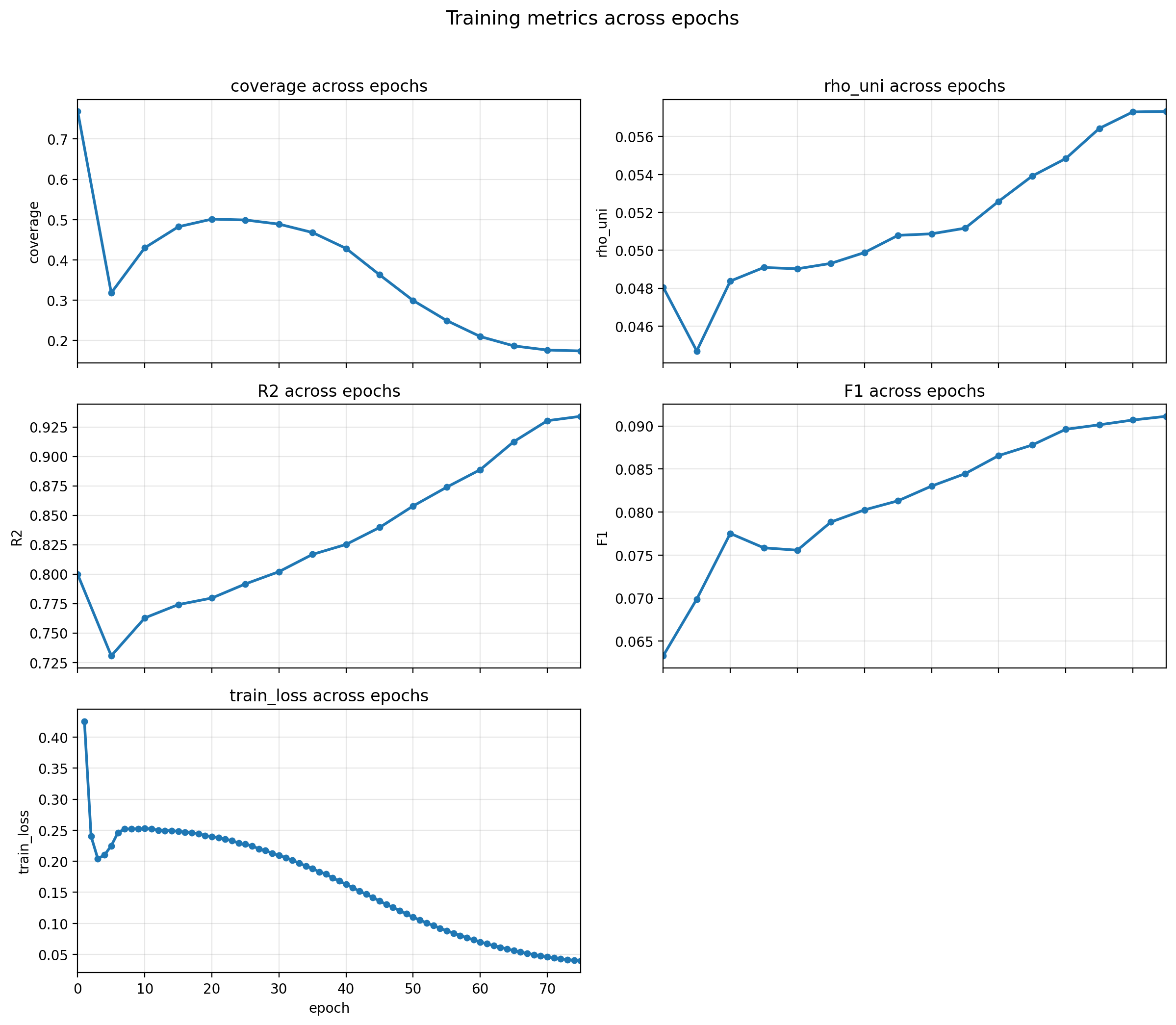}
\end{center}
\vspace{-4mm}
\caption{
Evolution of Coverage, $\rho_{uni}$, $R^2$, $F1$ and Loss of the \texttt{TopK} SAE across epochs of the training from epoch 0 to 75. 
}
\label{fig:Metrics_evolution}
\end{figure*}

\section{Application: analysis of the SAE representation on Husky Vs Wolf dataset}\label{sec:Application}

To further investigate how SAEs and CBMs encode semantic information, we conduct a focused experiment on a two-class subset of ImageNet composed of \emph{Husky} and \emph{Wolf} images, following the experimental setting proposed in~\citep{fel2023craft}. We rely on the annotated concept set provided in~\citep{chorna2025concept}, which offers image-level and class-level semantic attributes for these two classes. These annotations serve as ground-truth concepts for the CBM branch.

We begin by training a ResNet-50 backbone for binary classification on the Husky/Wolf subset. After training the classifier, we attach both a \texttt{TopK} SAE and a CBM module to the final convolutional block of the frozen backbone. The CBM branch predicts the concept vector through its concept-extractor matrix $W_c$, while the SAE learns a dictionary $D$ that provides a sparse generative decomposition of the features. This parallel construction allows us to analyze to what extent the learned SAE atoms align with 'human-based-words' CBM concepts.

\paragraph{Concept matching between SAE atoms and CBM concepts.}
To quantify semantic alignment between the SAE and the CBM, we compute a concept–dictionary similarity matrix $S$ defined as
\[
S_{i,j} = \mathrm{corr}(D_j, W_{c,i}),
\]
where $D_j$ denotes the $j$-th dictionary atom of the SAE and $W_{c,i}$ the $i$-th concept direction learned by the CBM. For each atom $j$, we identify its most associated concept by selecting
\[
\hat{i}(j) = \arg\max_i |S_{i,j}|.
\]
This establishes a direct one-to-one correspondence between SAE atoms and CBM concepts. In addition to matching dictionary atoms, we also align the latent representations of the SAE with CBM concepts. For each latent embedding $\mathbf{Z}^{\text{SAE}}_{\cdot j}$, we first identify the most influential atom by selecting the dictionary element that exhibits the strongest activation for that sample. Using the atom-to-concept assignment described above, we then propagate this match to the latent representations. This yields a consistent alignment across three elements: the dictionary atoms $D_j$, the CBM concept directions $W_{c,i}$, and the latent activations $\mathbf{Z}^{\text{SAE}}_{\cdot j}$. At this stage, both the SAE latent space and its dictionary are mapped onto the CBM semantic space, enabling a direct comparison between the representations learned by the two models.

\paragraph{UMAP visualization of SAE dictionary structure.}
Figure~\ref{fig:resultsUMAP} shows a UMAP projection of the SAE dictionary atoms, providing a global view of their geometric organization. To further analyze semantic alignment, we compute cosine similarities between selected CBM concepts and all SAE atoms. The UMAP plots reveal that only a small subset of atoms exhibits strong similarity to any given concept, while most atoms show weak or diffuse alignment. This indicates that a limited number of dictionary atoms capture the majority of the shared semantic signal between the SAE and the CBM, whereas many other atoms correspond to fine-grained or heterogeneous features that do not map cleanly onto human-annotated concepts.

\paragraph{Concept distribution across classes.}
To quantify how the two classes differ in terms of their concept activations, we assign to each sample the most influential concept determined via its highest correlation with the SAE dictionary as described above. This enables us to compute per-class histograms, shown in Figure~\ref{fig:histo}. These histograms highlight strong biases: certain concepts appear much more frequently in Husky images than in Wolf images, and vice versa. Such asymmetries echo known dataset biases in the Husky/Wolf classification task (e.g., wolves often appear in snowy or wild backgrounds whereas huskies appear in urban or domestic contexts).  
This demonstrates that the SAE+CBM analysis pipeline is able to reveal and disentangle these latent dataset biases without any explicit guidance, reinforcing the interpretability benefits of combining sparse generative modeling with concept supervision.

\begin{figure*}[t]
\vspace{-8mm}
\begin{center}
\begin{tabular}{ccc}
  % Ligne 1
  \includegraphics[width=0.29\linewidth]{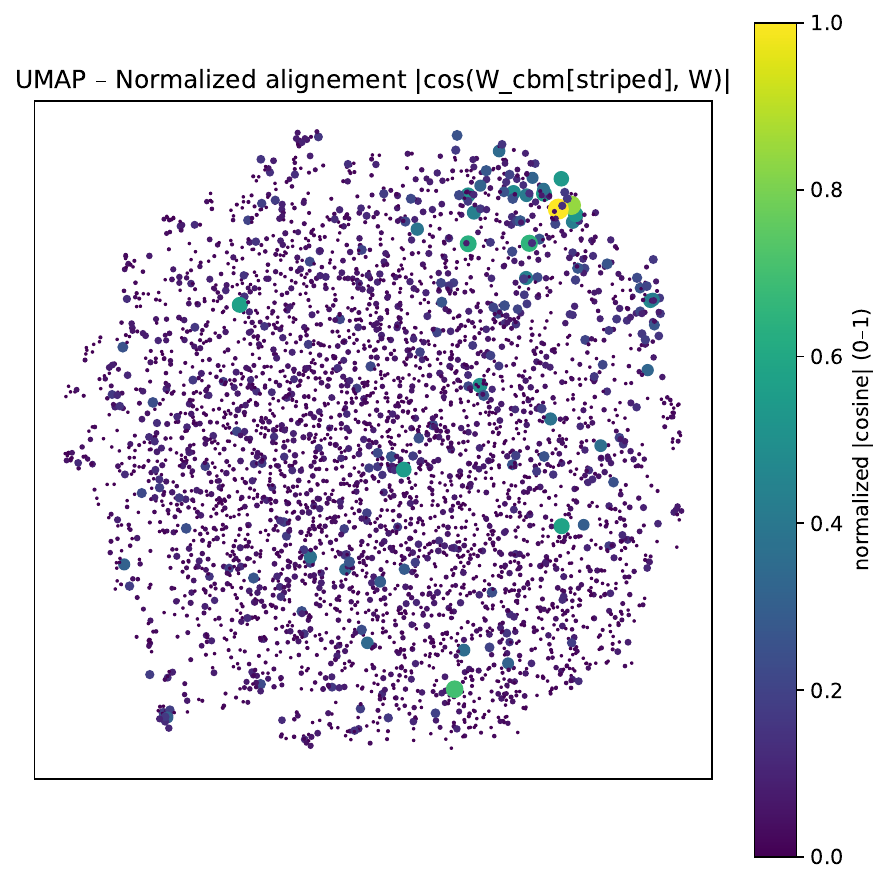} & 
  \includegraphics[width=0.29\linewidth]{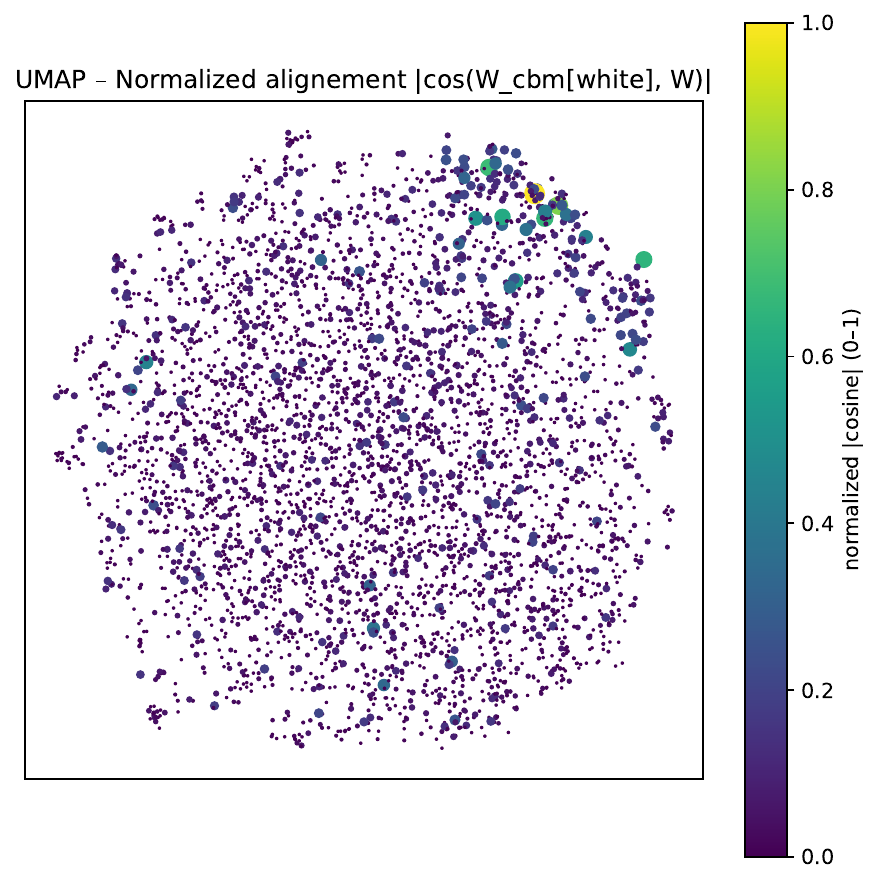} &
  \includegraphics[width=0.29\linewidth]{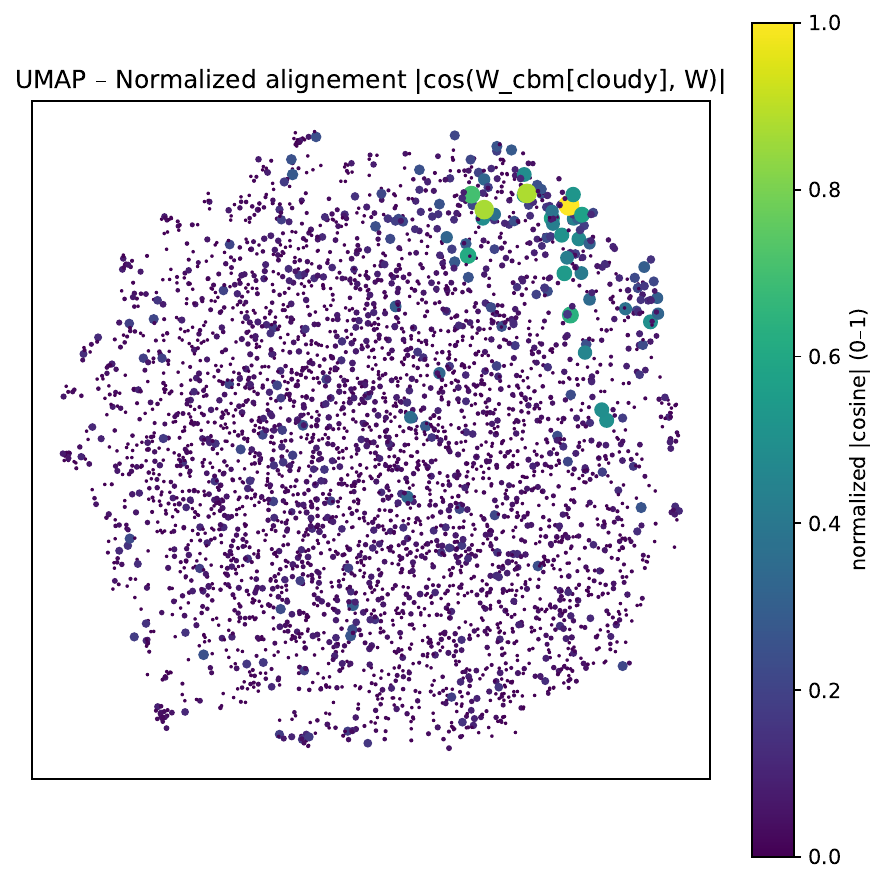} \\
  (a) concept ''striped'' & (b) concept ''white'' & (c) concept ''cloudy'' \\
  % Ligne 2
  \includegraphics[width=0.29\linewidth]{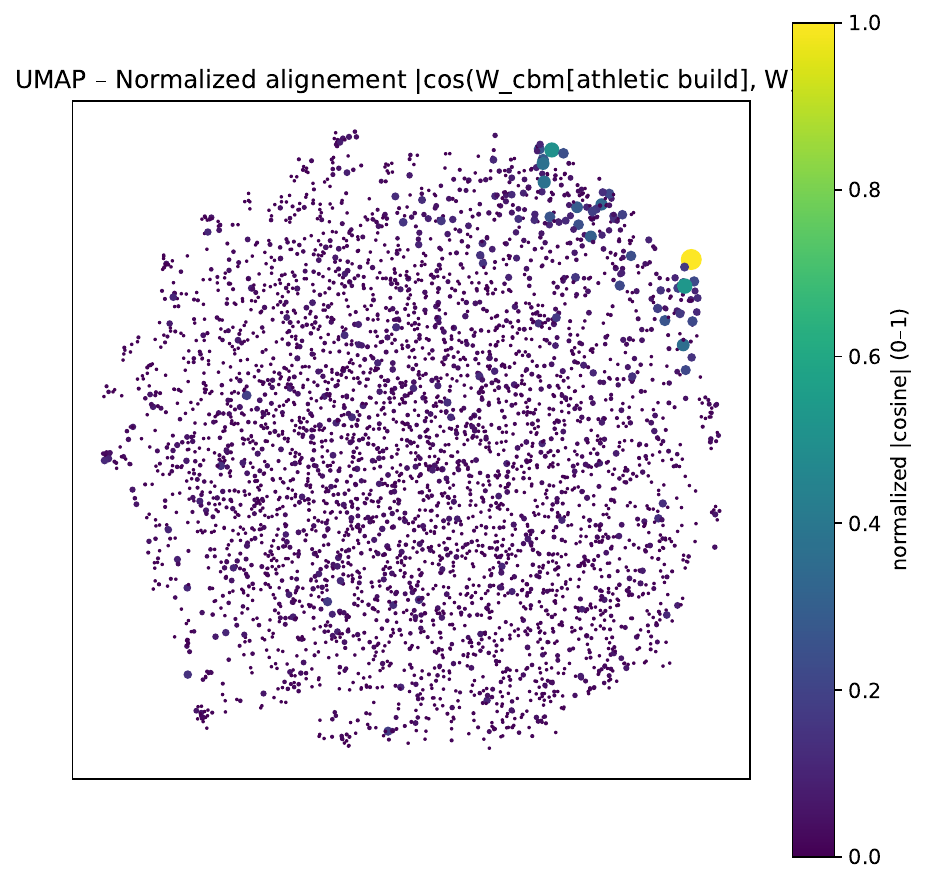} &
  \includegraphics[width=0.29\linewidth]{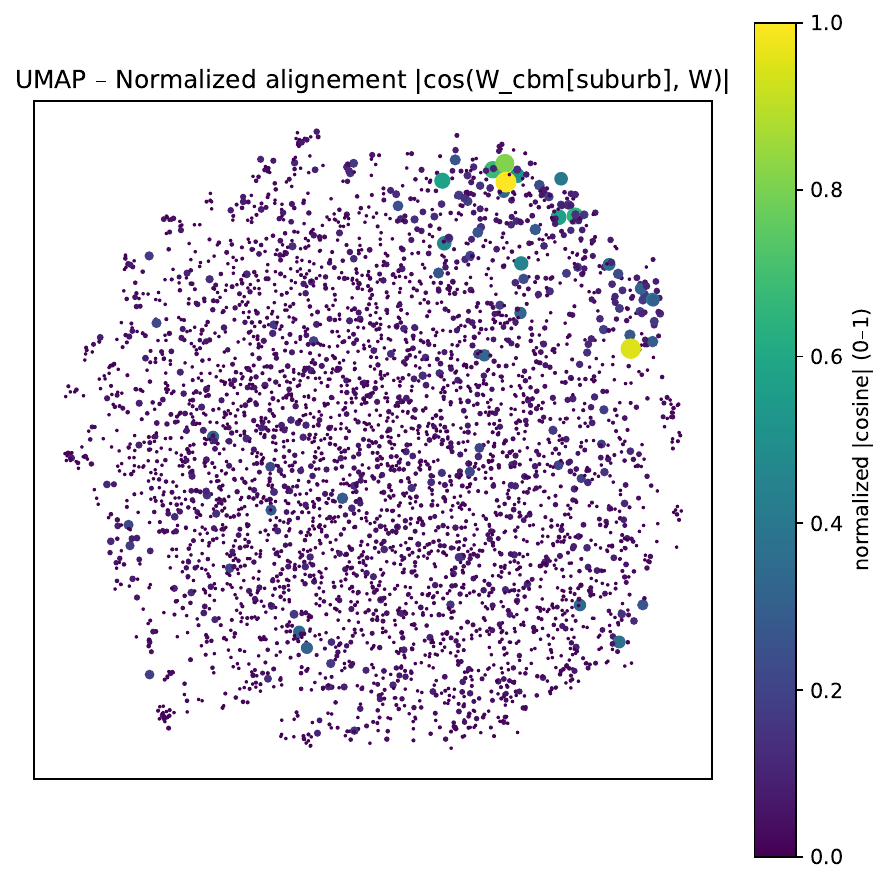} &
  \includegraphics[width=0.29\linewidth]{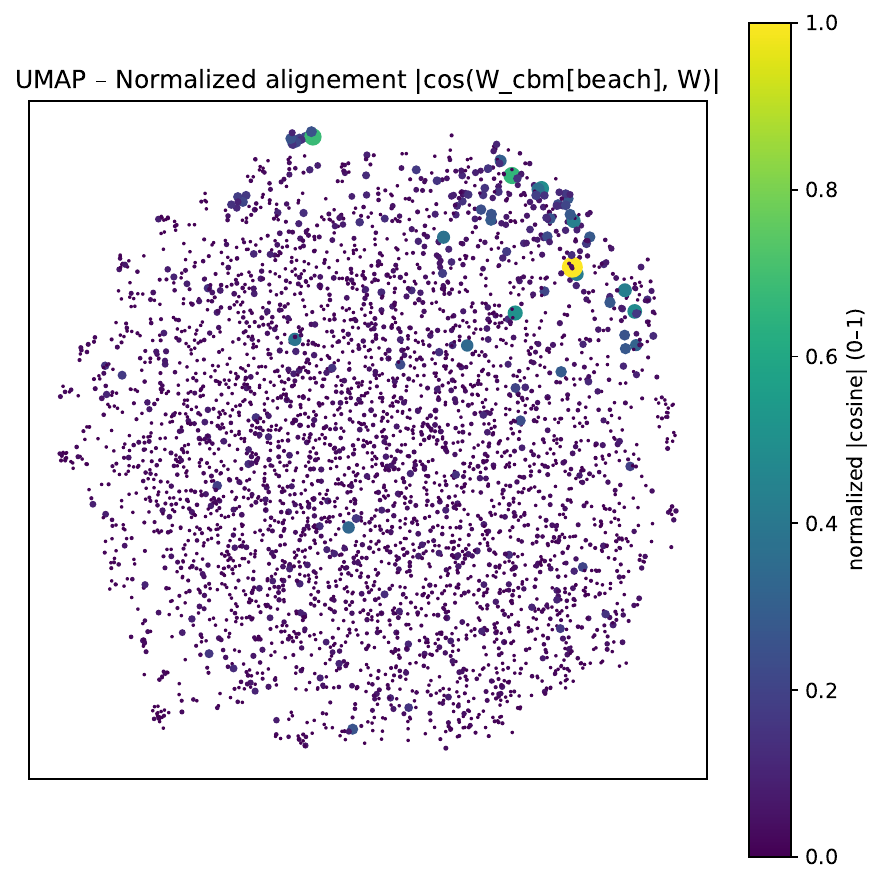}
  \\
  (d) concept ''atheltic build'' & (e) concept ''suburb' & (f) concept ''beach'' \\
  % Ligne 3
  \includegraphics[width=0.29\linewidth]{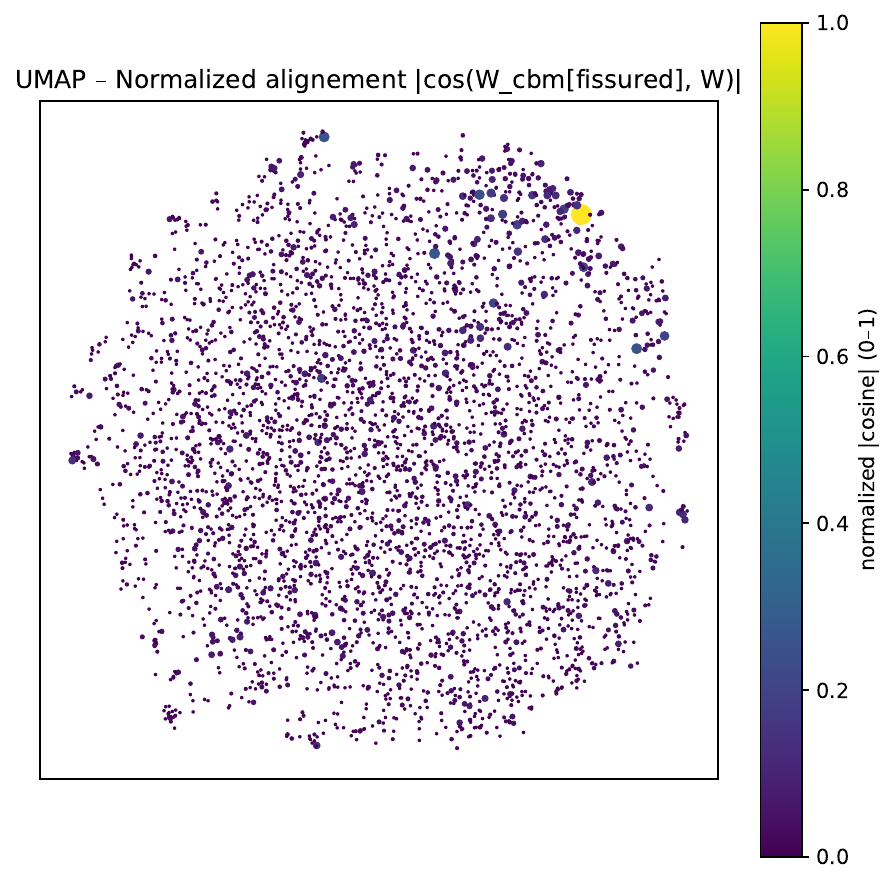} &
  \includegraphics[width=0.29\linewidth]{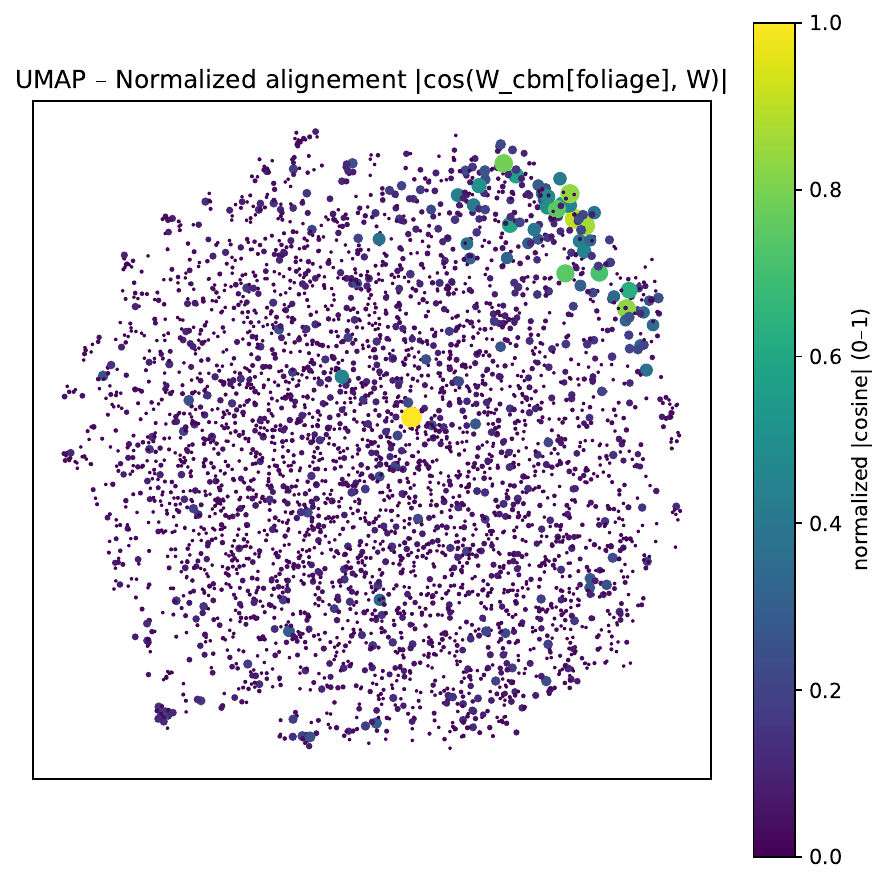} &
  \includegraphics[width=0.29\linewidth]{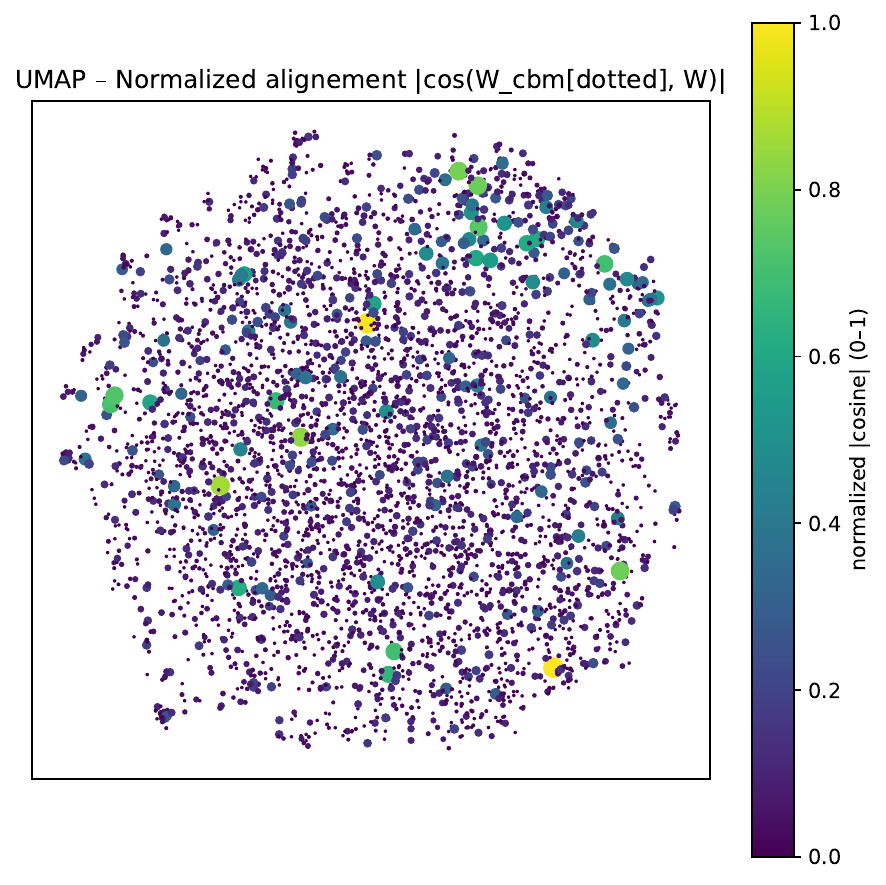}
  \\
  (g) concept ''fissured'' & (h) concept ''foliage' & (i) concept ''dotted'' \\
\end{tabular}
\end{center}
\caption{
Cosine similarities between selected CBM concepts and all SAE dictionary atoms, visualized through UMAP. Strong responses concentrate around a limited number of atoms, suggesting that only a small fraction of SAE directions encode concepts shared with the CBM.
}
\label{fig:resultsUMAP}
\end{figure*}

\begin{figure*}[t]
\vspace{-5mm}
\begin{center}
\includegraphics[width=1\linewidth]{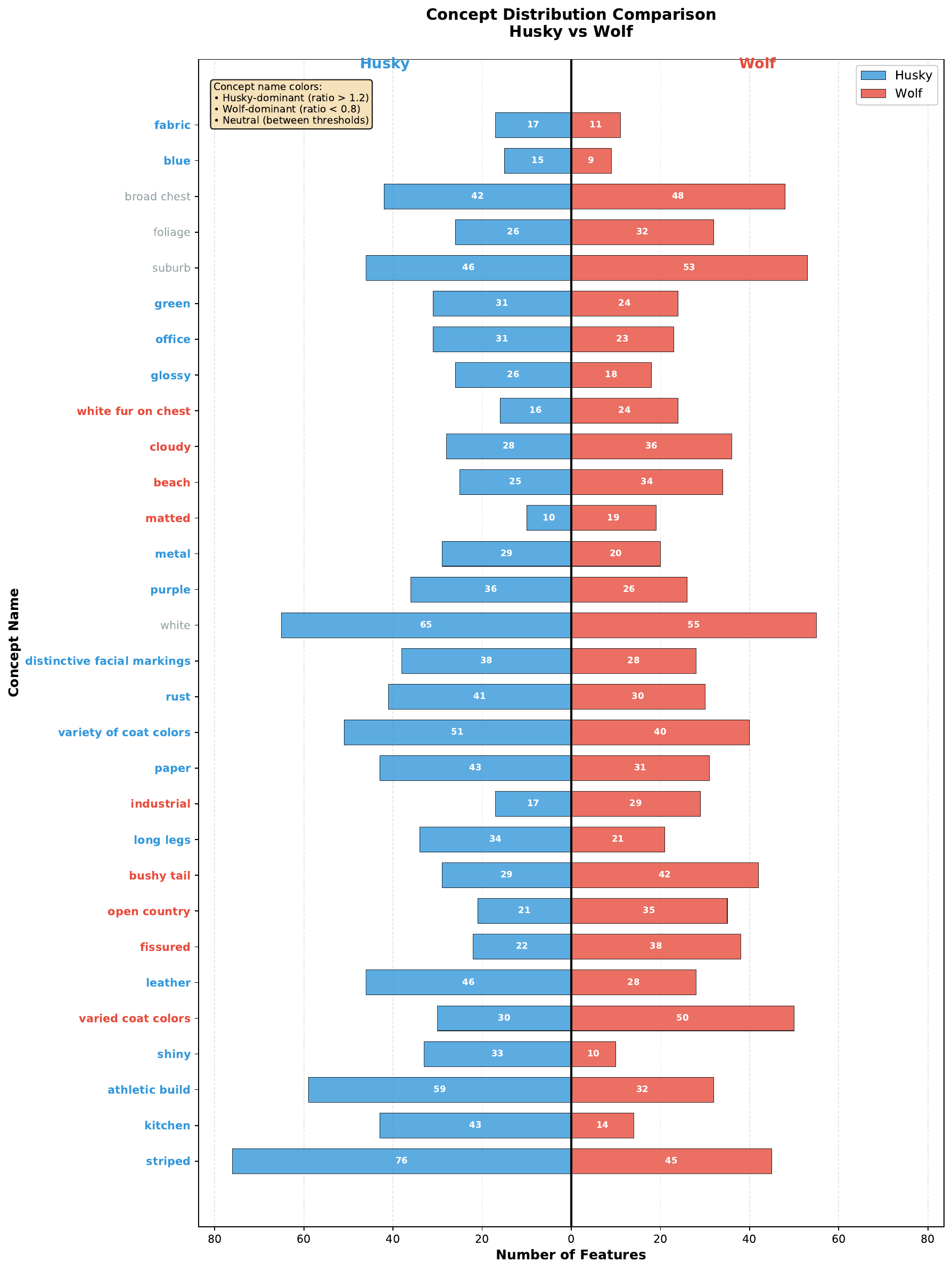}
\end{center}
\vspace{-4mm}
\caption{
Class-wise concept histograms for Husky and Wolf images. Each sample is assigned the concept with maximum correlation to its SAE representation, revealing clear concept-frequency biases between the two classes.
}
\label{fig:histo}
\end{figure*}

 \section{Effect of the number of CBM concepts on SAE metrics on Husky Vs Wolf dataset}
\label{sec:ApplicationV2}

In this experiment, we study how the number of concepts used in a CBM influences the evaluation metrics of a fixed SAE. All experiments are conducted on the \textit{Husky vs.\ Wolf} dataset with the concepts from ~\citep{chorna2025concept}, a standard benchmark for analyzing spurious correlations and classifier bias.

\paragraph{Experimental setup.}
We first train a ResNet-50 backbone on the Husky vs.\ Wolf classification task. On top of this frozen backbone, we train multiple CBMs that differ only in the number of concepts they use (70, 35, and 10). 70 is the official number of concepts and to reduce we discard the less frequent concepts in each class. Importantly, the SAE is trained once and kept fixed across all experiments, with a latent dimension of 4096. For each CBM configuration, we evaluate the same SAE using our suite of alignment, sparsity, coverage, and downstream performance metrics.

\begin{table}[t]
\centering
\caption{Effect of the number of CBM concepts on SAE (\texttt{TopK} SAE) metrics (Husky vs.\ Wolf, ResNet-50 backbone).}
\label{tab:nb_concepts}
\resizebox{\columnwidth}{!}{
\begin{tabular}{l c c c c c c c c c c}
\toprule
\textbf{SAE Type} & \textbf{\# CBM Concepts} & \textbf{Dim. SAE} & \textbf{Backbone} & \textbf{Sparsity} & $\boldsymbol{\rho_{\text{geom}}}$ & $\boldsymbol{\rho_{\text{act}}}$ & \textbf{Coverage} & \textbf{CBM Acc.} \\
\midrule
\texttt{TopK}  & 70 & 4096 & ResNet-50 & 0.9542 & \textbf{0.0757} & \textbf{0.0773} & \textbf{0.7125} & \textbf{95\%} \\
\texttt{TopK}  & 35 & 4096 & ResNet-50 & 0.9525 & 0.0628 & 0.0658 & 0.5659 & 93\% \\
\texttt{TopK}  & 10 & 4096 & ResNet-50 & 0.9523 & 0.0500 & 0.0551 & 0.5719 & 80\% \\
\bottomrule
\end{tabular}
}
\end{table}

\paragraph{Discussion.}
Table~\ref{tab:nb_concepts} highlights a clear relationship between the number of CBM concepts and the SAE evaluation metrics. As the number of CBM concepts increases, we observe consistent improvements in both geometric and activation-based alignment metrics ($\rho_{\text{geom}}$, $\rho_{\text{act}}$). This indicates that a richer concept vocabulary allows the CBM to better align with the structured representations learned by the SAE.

Coverage also increases substantially with more CBM concepts, reflecting a greater fraction of SAE concepts being meaningfully utilized by the CBM. In contrast, sparsity remains largely unchanged across configurations, as it is an intrinsic property of the SAE and independent of the CBM design.

Finally, CBM classification accuracy improves markedly as the number of concepts increases, suggesting that higher concept capacity enables the CBM to better capture discriminative features while maintaining interpretability.

%\clearpage
%\input{checklist.tex}

\end{document}